\documentclass{article}

\usepackage[preprint]{neurips_2026}

\usepackage[utf8]{inputenc}  
\usepackage[T1]{fontenc}     
\usepackage{url}             
\usepackage{xcolor}          
\usepackage{microtype}       

\usepackage{amsmath}
\usepackage{amssymb}
\usepackage{amsfonts}
\usepackage{mathtools}
\usepackage{amsthm}
\usepackage{bm}              
\usepackage{nicefrac}        

\usepackage{booktabs}        
\usepackage{multirow}
\usepackage{colortbl}

\usepackage{graphicx}
\usepackage{subcaption}
\usepackage{wrapfig}

\usepackage{algorithm}
\usepackage{algpseudocode}
\let\REQUIRE\Require
\let\ENSURE\Ensure
\let\STATE\State
\let\FOR\For
\let\ENDFOR\EndFor
\let\IF\If
\let\ENDIF\EndIf

\let\COMMENT\Comment
\usepackage[normalem]{ulem}

\usepackage[textsize=tiny]{todonotes}

\usepackage{hyperref}
\usepackage[capitalize,noabbrev]{cleveref}

\theoremstyle{plain}
\newtheorem{theorem}{Theorem}[section]

\theoremstyle{definition}

\theoremstyle{remark}

\newcommand{\twolinecap}[1]{%
  \parbox[t]{\linewidth}{\centering\footnotesize #1\vphantom{Ag}\\\vphantom{Ag}}%
}
 
\title{Refining Multidimensional Video Reward Models via Disentangled Influence Functions}

\author{%
  Muyao Wang \\
  The University of Tokyo
  \And
  Zeke Xie \\
  HKUST (Guangzhou)
  \And
  Hideki Nakayama \\
  The University of Tokyo
}

\begin{document}

\maketitle

\begin{abstract}
As Text-to-Video (T2V) generation models continue to evolve, the complexity of video evaluation necessitates a fine-grained assessment across various axes. 
To address this, recent works have focused on developing Multidimensional Video Reward Models (MVRMs), which decompose the evaluation process to better align with the multifaceted nature of human visual perception. However, training effective MVRMs is fundamentally challenged by the complex nature of video data. In this work, we identify a critical phenomenon termed \textit{Dimensional Heterogeneity}:  the reliability of a training sample can vary substantially across evaluation dimensions, meaning that a sample may provide reliable supervision for one objective while inducing high supervision risk for another. Consequently, prevailing data-centric methods that filter based on global scalar metrics are ill-posed for T2V tasks. To address this, we propose a disentangled influence framework that that efficiently estimates dimension-specific supervision risk. Leveraging this framework, we introduce two dimension-disentangled refinement strategies: Dimension-Disentangled Pruning, which removes extreme high-risk samples, and Dimension-Disentangled Reweighting, which softly down-weights high-risk supervision. Extensive experiments demonstrate that our disentangled strategies significantly outperform global filtering baselines, yielding reward models with superior alignment to ground truth.
\end{abstract}

\section{Introduction}

The domain of Text-to-Video (T2V) generation has witnessed a meteoric rise, with recent foundation models~\citep{wan2025wan,team2025kling,kong2024hunyuanvideo,blattmann2023stable,hong2022cogvideo,wang2023modelscope} demonstrating unprecedented capabilities in synthesizing dynamic visual content.
To further enhance generation quality and align these models with human intent, Reinforcement Learning from Human Feedback (RLHF) has been widely adopted as the gold standard, following its immense success in Large Language Models (LLMs)~\citep{stiennon2020learning,ziegler2019fine,christiano2017deep}.
However, applying RLHF to the video domain presents unique challenges. Unlike language tasks where a single scalar reward often suffices, video generation intrinsically involves multiple orthogonal dimensions~\citep{liu2024evalcrafter,huang2024vbench,liu2023fetv}---such as visual quality, temporal consistency, and semantic alignment---which cannot be adequately captured by a single metric. 
Consequently, \textbf{Multidimensional Video Reward Models (MVRMs)}~\citep{tongmj,he2024videoscore,xu2024visionreward,liu2025improving} have emerged as the standard for providing comprehensive feedback in video generation tasks.

However, training effective MVRMs is fundamentally bounded by the complex nature of the training data. 
Prior data-centric studies that aim to improve the quality of reward model~\cite{min2025understanding} typically utilize influence functions to filter data based on a single aggregated scalar (e.g., total loss), and primarily validate their effectiveness on NLP tasks.
We argue that this ``one-size-fits-all'' approach overlooks a critical characteristic of video data: \textbf{Dimensional Heterogeneity}. 
A training sample is rarely uniformly ``reliable'' or ``unreliable'' across all evaluation perspectives; it may provide a stable supervision signal for one dimension while inducing high supervision risk for another, potentially due to label noise.
As illustrated in our empirical analysis (Fig.~\ref{fig:scatter_plot}), we visualize the \textit{self-influence} distribution of training samples across different dimensions (e.g., \textit{Visual Quality} vs. \textit{Dynamic Degree}), where self-influence is used as a proxy for dimension-specific supervision risk.
\begin{wrapfigure}{r}{0.50\textwidth}
    \vspace{-10pt}
    \centering
    \includegraphics[width=0.48\textwidth]{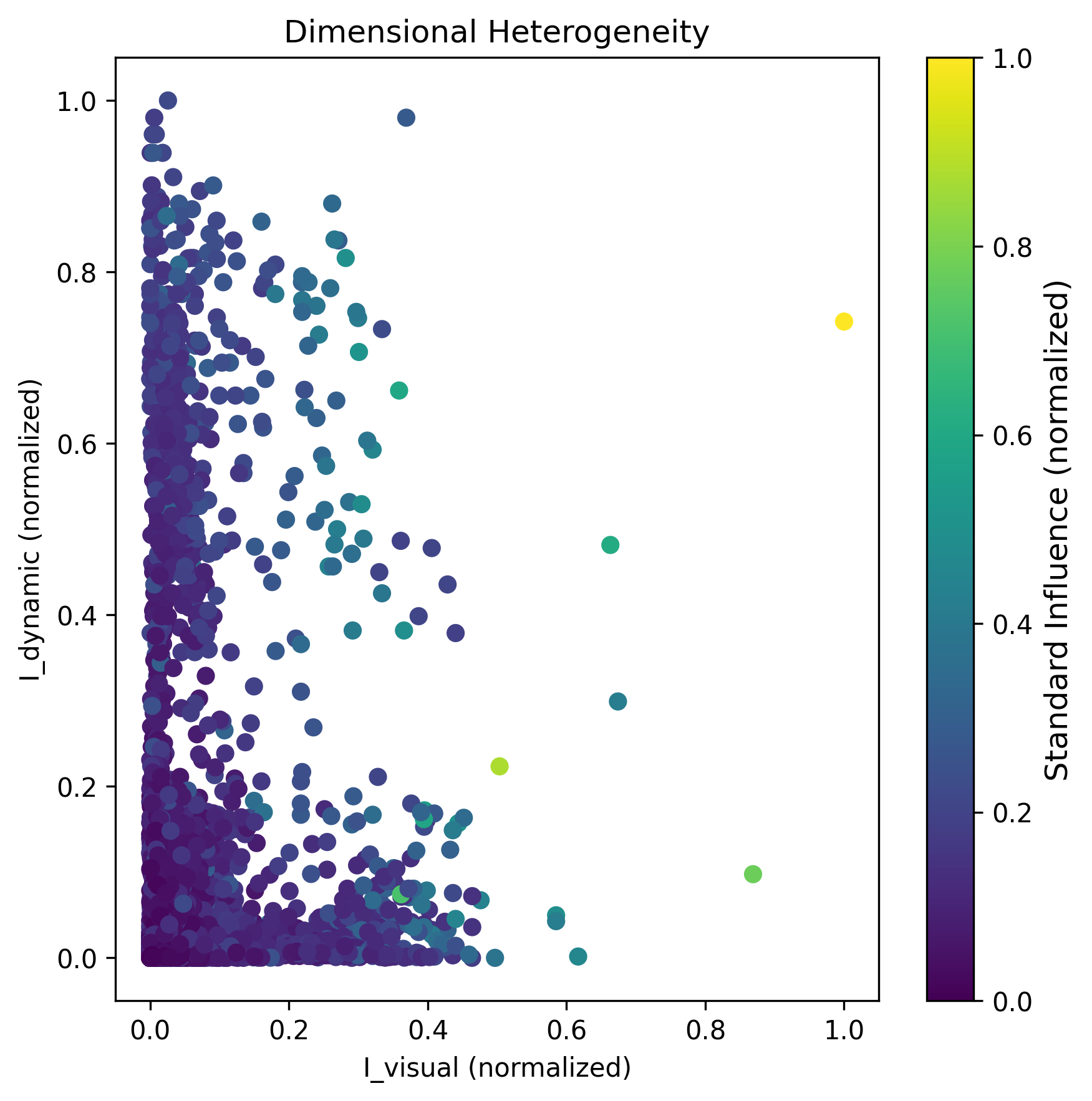}
    \caption{\textbf{Dimensional Heterogeneity in Video Training Data.} 
    We visualize the distribution of sample self-influence scores for \textit{Visual Quality} (x-axis) versus \textit{Dynamic Degree} (y-axis). 
The scatter plot exhibits a dispersed distribution with negligible correlation, indicating that supervision reliability is dimension-specific. 
Samples with high self-influence in one dimension, which we use as a proxy for dimension-specific supervision risk, often have low self-influence in the other.}
    \label{fig:scatter_plot}
    \vspace{-10pt}
\end{wrapfigure}
The results reveal a striking lack of correlation---samples with high self-influence, which we interpret as high dimension-specific supervision risk, in the visual dimension often exhibit negligible self-influence in the dynamic dimension.
This observation implies that indiscriminately filtering or reweighting data based on a global metric can lead to sub-optimal outcomes: it may retain samples that induce unreliable supervision, such as label noise for specific dimensions.

To address this challenge, we adopt a disentangled perspective to refine the training of MVRMs. 
Existing influence-based methods~\cite{koh2017understanding, pruthi2020estimating} operate on coarse-grained aggregated losses, failing to \textit{disentangle} the specific impacts of a sample on distinct reward dimensions. 
To bridge this gap, we propose a disentangled influence framework tailored for MVRMs. 
For practical scalability, we employ TracIn~\cite{pruthi2020estimating} strategy. 
This allows us to estimate dimension-specific self-influence using gradients on the corresponding projection heads, without requiring explicit Hessian inversion.
Crucially, we operationalize this framework into two alternative risk-aware strategies based on the same dimension-specific supervision-risk signal: 
\textbf{Dimension-Disentangled Pruning (DDP)}, which applies hard filtering to the extreme high-risk tail to prioritize robustness, 
and \textbf{Dimension-Disentangled Reweighting (DDR)}, which softly adjusts the contribution of high-risk supervision to reduce its optimization dominance while preserving potentially useful hard samples.
Rather than assuming that high self-influence perfectly separates label noise from difficult-but-informative samples, these two strategies provide different robustness--data utilization trade-offs under this inherent ambiguity.
The empirical results suggest that accounting for dimensional heterogeneity can contribute to more reliable MVRMs.

This work makes three main contributions.
\begin{itemize}
\item \textbf{Empirical Discovery of Dimensional Heterogeneity:}
We systematically analyze data quality in video reward models and reveal a previously overlooked phenomenon: supervision reliability can be highly heterogeneous across reward dimensions.
This challenges global data-quality assumptions and exposes the limitation of global filtering in capturing dimension-specific supervision risk.
\item \textbf{Disentangled Influence Framework:}
Motivated by this finding, we propose a disentangled influence framework with diagonal projection to efficiently estimate dimension-specific supervision risk.
\item \textbf{Risk-Aware Data Refinement Strategies:}
We design two strategies, \textit{Dimension-Disentangled Pruning} and \textit{Dimension-Disentangled Reweighting}, which perform hard filtering and soft reweighting using the same dimension-specific risk signal.
Experiments show that they outperform global filtering baselines and improve reward-model alignment with human judgments.
\end{itemize}

\section{Related Works}
\paragraph{Multidimensional Reward Modeling for Video Generation}

The rapid evolution of T2V generation has necessitated robust mechanisms to align model outputs with human preferences. Unlike assessing static images, evaluating synthesized videos requires a holistic perspective that encompasses multiple dimensions, such as visual quality, temporal consistency, and text-video alignment~\citep{tongmj,he2024videoscore,huang2024vbench,xu2024visionreward,wu2024boosting}.

Early approaches primarily relied on single-dimensional metrics or an ensemble of separate models. For instance, FVD~\citep{unterthiner2018towards} focuses on distribution-level quality, while CLIP-Score~\citep{hessel2021clipscore} measures semantic alignment. However, deploying multiple independent models for reward labeling is computationally expensive and results in a fragmented evaluation pipeline. Such disjointed approaches lack the efficiency and coherence of a unified system, making them ill-suited for the scalable training of modern video generation models.

To address these limitations, recent works have shifted towards unified multidimensional reward models built upon Large Vision-Language Models (LVLM). Benefiting from the generalization capabilities of LVLM like Mantis~\citep{jiangmantis} or Video-LLaVA~\citep{lin2024video}, methods such as VideoScore~\citep{he2024videoscore} and MJ-VIDEO~\citep{tongmj} adopt a paradigm in which a single forward pass outputs scores across multiple criteria simultaneously. This paradigm not only improves inference efficiency but also aligns more closely with fine-grained human feedback.

\paragraph{Multi-objective Learning vs. Data-centric Refinement}

A substantial body of prior work addresses multi-objective supervision through \emph{optimization-level balancing} techniques, including uncertainty-based loss weighting \citep{kendall2018multi}, gradient surgery methods such as PCGrad \citep{yu2020gradient}, and random or adaptive weighting schemes \citep{linreasonable}. These methods focus on mitigating gradient conflicts during joint optimization by dynamically adjusting objective contributions.

However, such approaches typically assume that the underlying training supervision is uniformly reliable across objectives. They do not explicitly model \emph{sample-level heterogeneity}, where a single training example may provide stable supervision for one dimension while inducing high-risk supervision for another, such as label noise. In the context of video reward modeling---where annotations are often heuristic, subjective, or weakly defined---this assumption is frequently violated. Our work departs from optimization-centric balancing and instead adopts a \emph{data-centric perspective}, aiming to diagnose and refine dimension-specific supervision risk.

\paragraph{Influence Functions for Supervision-Risk Estimation}

Influence functions provide a principled framework for estimating the effect of individual training samples on model predictions, while self-influence has been used as a useful signal for identifying potentially mislabeled samples \citep{koh2017understanding, pruthi2020estimating}. Due to the prohibitive cost of exact Hessian inversion, practical influence analysis often relies on approximations such as TracIn \citep{pruthi2020estimating}, last-layer influence estimation \citep{tan2024data,barshan2020relatif,park2023trak}, or Kronecker-factored approximations \citep{george2018fast}. These techniques have been successfully applied to data pruning and cleaning, primarily in single-objective settings.

Despite their effectiveness, existing influence-based methods fundamentally assume a \emph{scalar loss formulation}. When applied to MVRMs, influence is computed with respect to an aggregated objective, implicitly conflating the effects of distinct reward dimensions. As a result, dimension-specific supervision risks may be masked: a sample that appears benign under the aggregated objective may still induce unreliable supervision for a particular dimension. This limitation renders standard influence metrics ill-suited for fine-grained data refinement in multidimensional reward modeling.

\paragraph{Positioning of Our Work}

In this work, we bridge the gap between multidimensional reward modeling and data-centric influence analysis. Rather than treating influence as a monolithic scalar quantity, we explicitly disentangle sample influence across reward dimensions, enabling dimension-disentangled pruning and reweighting strategies. Our approach addresses a distinct and underexplored challenge: \emph{dimension heterogeneity in MVRMS training}. To our best knowledge, this is the first systematic investigation leveraging disentangled influence framework for training MVRMS.

\section{Methodology}
\label{sec:method}

In this section, we formulate the \textbf{Disentangled Influence Framework}, which serves as the theoretical foundation for our dimension-disentangled data refinement strategies. We demonstrate how the scalar influence metric used in prior works can be mathematically decomposed into a matrix form, revealing the underlying interplay between distinct reward dimensions.

\subsection{Formulation: Disentangled Influence Function}

Consider a multidimensional reward model $f_{\boldsymbol{\theta}}$ parameterized by $\boldsymbol{\theta} \in \mathbb{R}^d$, designed to evaluate $K$ distinct dimensions (e.g., visual quality, textual alignment). Let $\mathcal{Z} = \{z_i\}_{i=1}^N$ be the training dataset. The training objective minimizes the empirical risk $\mathcal{R}(\boldsymbol{\theta})$, where the loss for a single sample is the weighted aggregation of $K$ dimension-specific losses:
\begin{equation}
    \mathcal{L}(z; \boldsymbol{\theta}) = \sum_{k=1}^K \lambda_k \mathcal{L}_k(z; \boldsymbol{\theta}).
\end{equation}

\textbf{The Scalar Limitation.}
Formally, the standard influence function $\mathcal{I}(z, z_{test})$ estimates the effect of upweighting a training sample $z$ on the loss of a test sample $z_{test}$, computed as $\mathcal{I}(z, z_{test}) := \nabla_{\boldsymbol{\theta}} \mathcal{L}(z_{test})^\top \mathbf{H}_{\hat{\boldsymbol{\theta}}}^{-1} \nabla_{\boldsymbol{\theta}} \mathcal{L}(z)$~\cite{koh2017understanding}.
While effective for single-objective tasks, this scalar formulation becomes problematic in multidimensional settings.
Crucially, the standard influence operates on the \textit{aggregated} loss gradients $\nabla_{\boldsymbol{\theta}} \mathcal{L}$.
This aggregation collapses the vector-valued supervision signals into a single scalar, obscuring the nuanced \textit{dimensional heterogeneity}---it fails to distinguish whether a training sample acts as a ``collaborator'' for the visual dimension while simultaneously being a ``competitor'' for the semantic dimension.

\textbf{Disentangled Influence Matrix.}
To resolve this ambiguity, we propose to analyze influence at the granularity of individual dimensions. By exploiting the linearity of the gradient operator, we prove that the monolithic scalar influence is, in fact, the sum of entries in a latent interaction matrix. We term this the \textbf{Disentangled Influence Matrix}.

\begin{theorem}
\label{thm:disentangled_influence}
\textbf{(Disentangled Influence Function)}
Let $\mathcal{I}(z, z_{test})$ be the standard scalar influence of training sample $z$ on test sample $z_{test}$. Assuming the Hessian $\mathbf{H}_{\hat{\boldsymbol{\theta}}}$ is positive definite, this scalar influence can be exactly decomposed into the sum of all elements of a $K \times K$ Disentangled Influence Matrix $\boldsymbol{\mathcal{M}}(z, z_{test}) \in \mathbb{R}^{K \times K}$:
\begin{equation}
    \mathcal{I}(z, z_{test}) = \sum_{j=1}^K \sum_{k=1}^K \phi_{j,k}, \quad \text{where } \boldsymbol{\mathcal{M}} = [\phi_{j,k}].
\end{equation}
Each entry $\phi_{j,k}$ represents the \textbf{Disentangled Influence} of the $k$-th dimension of the training sample on the $j$-th dimension of the test sample:
\begin{equation} \label{eq:matrix_element}
    \phi_{j,k} := \lambda_j \lambda_k \nabla_{\boldsymbol{\theta}} \mathcal{L}_j(z_{test}; \hat{\boldsymbol{\theta}})^\top \mathbf{H}_{\hat{\boldsymbol{\theta}}}^{-1} \nabla_{\boldsymbol{\theta}} \mathcal{L}_k(z; \hat{\boldsymbol{\theta}}).
\end{equation}
\end{theorem}

\begin{proof}
See Appendix~\ref{app:proof_thm1} for the detailed derivation based on the linearity of gradients.
\end{proof}

This decomposition converts influence estimation from a scalar quantity into a dimension-resolved attribution map at the sample level, making how training samples affect each reward dimension.

\subsection{Isolating Dimension-Specific Supervision Risk via Diagonal Projection}
\label{sec:self_influence}

While Theorem~\ref{thm:disentangled_influence} derives a complete $K \times K$ interaction matrix, we use a diagonal projection to obtain a dimension-specific supervision-risk signal.
We bypass the off-diagonal terms ($\phi_{j,k}, j \neq k$), which mainly reflect cross-dimensional optimization coupling, and focus on the diagonal term $\phi_{k,k}$.
This diagonal term, termed \emph{self-influence}, measures how strongly a sample affects the learning of its corresponding reward dimension.
Following prior influence-function studies, self-influence has been used to identify mislabeled or abnormal training examples~\citep{koh2017understanding,pruthi2020estimating}.
In our multidimensional setting, we interpret high self-influence as indicating high supervision risk for a specific reward dimension, which may arise from noisy or hard data.

Formally, we define the Dimension-Specific Self-Influence $\mathcal{S}_k(z)$ as a proxy for the supervision risk of sample $z$ with respect to the $k$-th reward dimension. 
It is obtained by restricting the influence analysis to the diagonal term $\phi_{k,k}(z,z)$:
\begin{equation} 
\label{eq:self_influence}
    \mathcal{S}_k(z) := \phi_{k,k}(z, z)
    = \nabla_{\boldsymbol{\theta}} \mathcal{L}_k(z; \hat{\boldsymbol{\theta}})^\top 
    \mathbf{H}_{\hat{\boldsymbol{\theta}}}^{-1} 
    \nabla_{\boldsymbol{\theta}} \mathcal{L}_k(z; \hat{\boldsymbol{\theta}})
\end{equation}

\paragraph{Practical Approximation.}
While Eq.~\eqref{eq:self_influence} is expressed using the inverse Hessian for theoretical clarity, explicitly computing $\mathbf{H}_{\hat{\boldsymbol{\theta}}}^{-1}$ is intractable at scale.
In practice, we adopt a Hessian-free approximation~\citep{pruthi2020estimating}, which preserves the relative ranking of samples within each dimension for efficient pruning and reweighting, while reducing the complexity of dimension-wise influence estimation from $\mathcal{O}(K^2)$ to $\mathcal{O}(K)$.

\subsection{Dimension-Disentangled Pruning Strategy}
\label{sec:method_pruning}

\paragraph{Goal of Disentangled Pruning.} 
Standard pruning methods typically filter data based on a global scalar metric (e.g., total loss), implicitly assuming that supervision reliability is uniform across evaluation axes. 
However, \textit{Dimensional Heterogeneity} suggests that a sample may exhibit high supervision risk in a specific reward dimension while appearing reliable in others. 
We therefore propose \textbf{Dimension-Disentangled Pruning (DDP)}, a hard filtering strategy that removes samples in the extreme high-risk tail of any reward dimension.

Let $\mathcal{D} = \{z_i\}_{i=1}^N$ be the full dataset. DDP constructs a refined subset $\mathcal{D}_{\mathrm{refined}} \subset \mathcal{D}$:
\begin{equation} \label{eq:subset_risk}
    \min_{\boldsymbol{\theta}} \frac{1}{|\mathcal{D}_{\mathrm{refined}}|} 
    \sum_{z_i \in \mathcal{D}_{\mathrm{refined}}} 
    \sum_{k=1}^K \mathcal{L}_k(z_i; \boldsymbol{\theta})
\end{equation}
where $\mathcal{D}_{\mathrm{refined}} = \bigcap_{k=1}^K \mathcal{V}_k$, and $\mathcal{V}_k$ denotes samples not identified as high-risk for the $k$-th dimension.

\paragraph{Pruning Criterion via Disentangled Influence.}
Using the self-influence score $\mathcal{S}_k(z_i)$ in Eq.~\eqref{eq:self_influence}, we define the high-risk set for reward dimension $k$ as:
\begin{equation}
    \mathcal{R}_k = \left\{ z_i \in \mathcal{D} \mid \mathcal{S}_k(z_i) > \tau_k \right\},
\end{equation}
where $\tau_k$ is the top-$\rho$ percentile threshold of self-influence scores for dimension $k$.
A sample is pruned if it belongs to any dimension-wise high-risk set:
\begin{equation}
    \mathcal{R}_{\mathrm{total}} = \bigcup_{k=1}^K \mathcal{R}_k 
    = \left\{ z_i \mid \exists k, \mathcal{S}_k(z_i) > \tau_k \right\}.
\end{equation}
The final training dataset is $\mathcal{D}_{\mathrm{refined}} = \mathcal{D} \setminus \mathcal{R}_{\mathrm{total}}$.
This policy targets extreme dimension-specific supervision risk rather than attempting to explicitly separate label noise from difficult-but-useful samples.
The detailed procedure is summarized in Algorithm~\ref{alg:rdap}.

\subsection{Dimension-Disentangled Reweighting Strategy}
\label{sec:method_reweighting}

\paragraph{Goal of Dimension-Disentangled Reweighting.}
While DDP applies hard filtering to the extreme high-risk tail, it may discard difficult but informative samples when high self-influence does not necessarily correspond to label noise.
To provide a softer alternative, we propose \textbf{Dimension-Disentangled Reweighting (DDR)}, which continuously adjusts the contribution of high-risk supervision instead of removing samples.
Standard reweighting methods often normalize weights across dimensions within a sample, creating a ``zero-sum'' competition where down-weighting one dimension inadvertently up-weights others.
In contrast, DDR evaluates each reward dimension independently relative to its global self-influence distribution, so that the weight of each dimension is determined by its own supervision risk without interference from other dimensions.

\paragraph{Reweighting via Self-influence.}
We leverage the self-influence scores $\mathcal{S}_k(z_i)$ as a proxy for dimension-specific supervision risk.
To handle the varying scales of self-influence across different heads, we first perform a dimension-wise z-score normalization.
Let $\mu_k$ and $\sigma_k$ be the mean and standard deviation of the self-influence scores for dimension $k$ across the entire dataset. The normalized risk score $\hat{\mathcal{S}}_{i,k}$ is computed as:
\begin{equation}
    \hat{\mathcal{S}}_{i,k} = \frac{\mathcal{S}_k(z_i) - \mu_k}{\sigma_k + \epsilon}
\end{equation}
We then apply an independent sigmoidal decay function to map these scores to training weights $w_{i,k}$:
\begin{equation} \label{eq:sigmoid_weight}
    w_{i,k} = \frac{1}{1 + \exp\left( \frac{\hat{\mathcal{S}}_{i,k}}{\tau} \right)}
\end{equation}
where $\tau$ is a temperature hyperparameter controlling the sharpness of the transition.
This formulation imposes a monotonic penalty: dimensions with higher relative self-influence are smoothly down-weighted, reducing their optimization dominance while still allowing them to contribute to training.
Finally, to maintain the effective learning rate, we rescale the weight matrix $\mathbf{W}$ such that the global mean weight equals 1.0.
The implementation of DDR is detailed in Algorithm~\ref{alg:rdar}.

\section{Experiments}
\label{sec:experiments}

In this section, we evaluate the proposed Disentangled Influence Function framework through two dimension-resolved applications: influence-guided pruning for filtering samples that negatively affect specific evaluation dimensions, and dimension reweighting for modulating sample contributions under heterogeneous multi-dimension supervision.

\textbf{Base Model and Architecture.}
We construct our framework upon VideoScore~\citep{he2024videoscore}, which utilizes Mantis-Idefics2-8B~\citep{jiangmantis} as the backbone.
Unlike the generative scoring mode, we adopt the \textbf{regression scoring mode}~\citep{he2024videoscore} where a linear regression head maps the backbone features to 5 distinct scalar rewards. 
Our influence computation focuses on the parameters of this regression head to ensure computational efficiency while capturing dimension-specific semantics.

\textbf{Experimental Setup and Implementation.}
To comprehensively evaluate our methods, we compare against two categories of baselines: (1) \textbf{Adaptive Balancing Strategies} (including Uncertainty Weighting~\citep{kendall2018multi}, RLW~\citep{linreasonable}, and PCGrad~\citep{yu2020gradient}), which balance optimization at the objective level; and (2) \textbf{Data Pruning Strategies} based on per-dimension loss, per-dimension uncertainty~\citep{harrison2024variational}, and TracIn Influence~\citep{pruthi2020estimating}, where the latter serves as a ablation to validate the necessity of disentangled analysis. The dataset and training details can be found in Appendix~\ref{traindetails}.

\begin{table*}[h!]
    \centering
    \caption{\textbf{Main Results: Method Comparison across Five Dimensions.} 
    We report Spearman correlation ($\times 100$) on the test set across five evaluation dimensions. 
    For all pruning experiments, we report results using the optimal pruning ratio selected via validation dataset. (Appendix~\ref{app:pruning_ratio_selection}). 
    Bold numbers indicate the best performance within each column.}
    \label{tab:main_results_merged}
    
    \setlength{\tabcolsep}{5pt}
    \renewcommand{\arraystretch}{0.95}

    \resizebox{0.95\textwidth}{!}{
    \begin{tabular}{l l ccccc c}
        \toprule
        \multicolumn{2}{c}{\multirow{2}{*}{\textbf{Method}}} & \multicolumn{5}{c}{\textbf{Dimension Performance}} & \multirow{2}{*}{\textbf{Avg.}} \\
        \cmidrule(lr){3-7}
        \multicolumn{2}{c}{} & \textbf{Visual} & \textbf{Temporal} & \textbf{Dynamic} & \textbf{Text} & \textbf{Factual} & \\
        \midrule
        
        \textbf{Baseline} & Equal Weighting & 81.16 & 73.55 & 57.08 & 50.61 & 78.91 & 68.26 \\
        \midrule
        
        \multirow{3}{*}{\shortstack[l]{\textbf{Adaptive}\\\textbf{Balancing}}} 
        & Uncertainty Weighting \cite{kendall2018multi} & 81.47 & 74.15 & 56.70 & 49.60 & 79.15 & 68.21 \\
        & RLW \cite{linreasonable} & 82.05 & 73.63 & 56.80 & 48.80 & 78.58 & 67.97 \\
        & \textbf{Ours (DDR)} & \textbf{82.24} & \textbf{74.24} & \textbf{57.59} & \textbf{52.49} & \textbf{80.36} & \textbf{69.38} \\
        \midrule
        
        \multirow{4}{*}{\shortstack[l]{\textbf{Data}\\\textbf{Pruning}}} 
        & TracIn \cite{pruthi2020estimating} & 80.46 & 72.93 & 57.24 & 50.18 & 78.20 & 67.80 \\
        & High-Loss & 81.36 & 74.15 & 57.93 & 50.76 & 79.18 & 68.67 \\
        & Uncertainty & 80.07 & 73.81 & 55.34 & 50.81 & 77.41 & 67.49 \\
        & \textbf{Ours (DDP)} & \textbf{82.62} & \textbf{75.09} & \textbf{59.26} & \textbf{51.58} & \textbf{79.28} & \textbf{69.57} \\
        \bottomrule
    \end{tabular}
    }
    
    \begin{minipage}{0.95\textwidth} 
        \vspace{1mm} 
        \footnotesize 
        \textit{Note:} PCGrad~\citep{yu2020gradient} is not reported due to GPU memory constraints arising from explicit per-dimension gradient computation at scale.
    \end{minipage}
    
\end{table*}

\subsection{Main Results}
\label{subsec:main_results}

Table~\ref{tab:main_results_merged} presents a comprehensive comparison across five evaluation dimensions, including \textit{Visual Quality}, \textit{Temporal Consistency}, \textit{Dynamic Degree}, \textit{Text-to-Video Alignment}, and \textit{Factual Consistency}, measured by Spearman correlation ($\times 100$) on the test set.

\textbf{Effectiveness of DDR (Reweighting).}
As shown in Table~\ref{tab:main_results_merged}, \textbf{DDR} achieves the best average performance among adaptive balancing methods.
Unlike objective-level balancing strategies, which treat each reward dimension as a whole, DDR models sample-level variability within each dimension.
By softly down-weighting samples with high dimension-specific supervision risk, DDR improves the reward model across all five axes.

\textbf{Effectiveness of DDP (Pruning).}
\textbf{DDP} achieves the highest average Spearman correlation among all compared methods while maintaining competitive performance on each dimension.
Compared with heuristic dimension-aware pruning methods such as \textit{Loss-based} and \textit{Uncertainty-based} filtering, DDP provides more consistent gains.
This suggests that self-influence offers a more informative supervision-risk signal.

\textbf{Ablation Study: The Necessity of Disentanglement.}
To isolate the effect of disentanglement, we compare our method with TracIn using the same pruning protocol.
The only difference is the pruning metric: TracIn ranks samples by scalar self-influence, whereas our method ranks them by dimension-wise self-influence.
All other experiment settings are kept identical.
The improvement of our method therefore indicates that disentangled supervision-risk estimation, rather than additional pruning heuristics, accounts for the gain.

\textbf{Ablation Study: Diagonal Projection.}
The full disentangled influence matrix contains both diagonal and off-diagonal terms.
However, for data refinement, we require a clear scalar supervision-risk score for each reward dimension.
We therefore compare our diagonal-only criterion with a per-row-sum variant that aggregates both diagonal and off-diagonal terms for each dimension.
\begin{wraptable}{r}{0.46\textwidth}
    \centering
    \vspace{-10pt}
    \caption{
    Ablation of using per-row-sum influence, which aggregates both diagonal and off-diagonal terms in the disentangled influence matrix.
    Spearman correlations are reported.
    }
    \label{tab:row_sum_ablation}
    \resizebox{0.46\textwidth}{!}{
    \begin{tabular}{lccccc}
        \toprule
        \textbf{Method} & \textbf{VQ} & \textbf{TC} & \textbf{DD} & \textbf{TVA} & \textbf{FC} \\
        \midrule
        DDR w/ row-sum & 81.91 & 73.75 & 57.74 & 50.56 & 78.45 \\
        DDP w/ row-sum & 81.86 & 74.77 & 59.07 & 51.45 & 78.57 \\
        \bottomrule
    \end{tabular}
    }
    \vspace{-10pt}
\end{wraptable}
As shown in Table~\ref{tab:row_sum_ablation}, the per-row-sum variant performs worse than our diagonal-only design in Table~\ref{tab:main_results_merged}.
This indicates that off-diagonal terms do not provide a more effective signal for pruning or reweighting.
We attribute this to the fact that off-diagonal terms primarily reflect cross-dimensional optimization coupling under shared parameters, while diagonal self-influence more directly captures the supervision risk of the reward dimension.

\begin{wraptable}{r}{0.52\textwidth}
    \vspace{-10pt}
    \centering
    \caption{\textbf{GenAI-Bench Performance.} We evaluate generalization using Pairwise Accuracy (\%).}
    \label{tab:genai_results}
    \resizebox{0.50\textwidth}{!}{
    \begin{tabular}{lccc}
        \toprule
        \textbf{Metric} & \textbf{Original} & \textbf{DDP} & \textbf{DDR} \\
        \midrule
        GenAI-Bench Acc. & 64.27 & \textbf{70.16} & 69.22 \\
        \bottomrule
    \end{tabular}
    }
    \vspace{-5pt}
\end{wraptable}
\textbf{GenaiBench}
To evaluate out-of-distribution robustness, we test on GenAI-Bench~\citep{jiang2024genai}, an external benchmark for video generation assessment.
As shown in Table~\ref{tab:genai_results}, both DDP and DDR improve pairwise accuracy over the original VideoScore baseline (64.27\% $\rightarrow$ 70.16\% and 69.22\%).
These results suggest that our disentangled data refinement improves reward-model generalization beyond the training distribution.

Please refer to Appendix~\ref{sec:generalization} for generalization analysis on various model architectures,~\ref{app:stability} for error bar experiment, and ~\ref{sec:efficiency} for efficiency analysis.

\subsection{Sensitivity Analysis}

\begin{figure*}[th!]
    \centering
    \includegraphics[width=\textwidth]{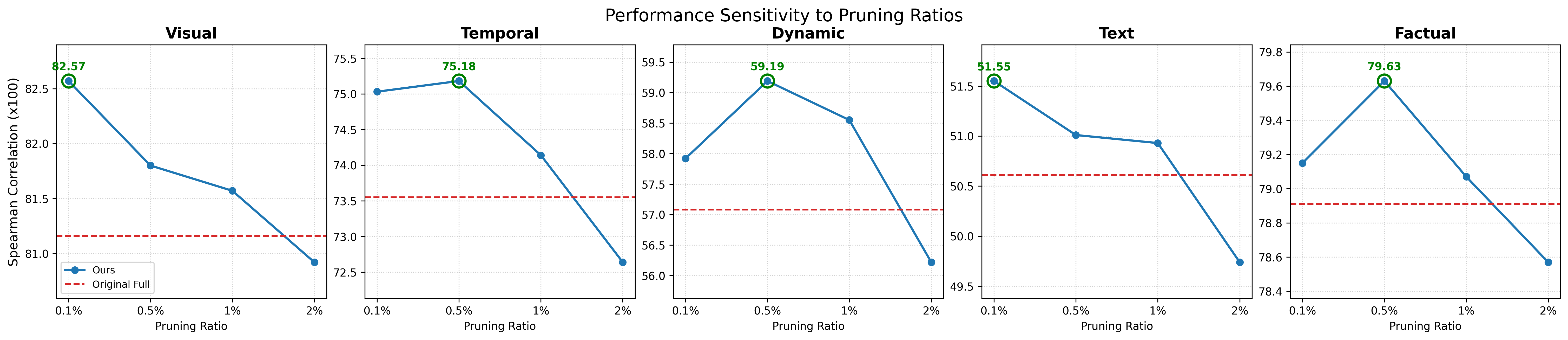} 
    \caption{We analyze the impact of removing the top-$k\%$ high-self-influence samples across five evaluation dimensions. 
    The \textcolor{blue}{\textbf{blue solid lines}} denote the performance of our method, while the \textcolor{red}{\textbf{red dashed lines}} represent the baseline performance of the model trained on the full original dataset. 
    \textcolor{green}{\textbf{Green circles}} highlight the peak performance for each dimension.
    Moderate pruning based on disentangled self-influence improves performance over the baseline, with peak ratios at 0.1\% for Visual/Text and 0.5\% for the other dimensions, while excessive pruning degrades performance.}
    \label{fig:pruning_curve}
\end{figure*}

\textbf{Sensitivity Analysis.}
Figure~\ref{fig:pruning_curve} shows the effect of different pruning ratios ($k \in \{0.1\%, 0.5\%, 1\%, 2\%\}$) on performance across reward dimensions. 
We observe clear dimension-dependent sensitivity: \textit{Visual Quality} and \textit{Text Alignment} peak at a conservative ratio ($k=0.1\%$), while \textit{Dynamic Degree}, \textit{Temporal Consistency}, and \textit{Factual Consistency} favor a more aggressive setting ($k=0.5\%$). 
When $k$ becomes too large, performance falls below the baseline (red dashed line), indicating that excessive pruning removes informative samples. 
These results support the non-uniform distribution of data quality across dimensions and highlight the utility of our method.

Please refer to Appendix~\ref{app:ddr_temperature} for the sensitivity experiment of hyperparameter $\tau$, ~\ref{app:ddp_overlap} for actual union removal size of DDP and ~\ref{app:pruning_ratio_selection} for the sensitivity analysis on divided validation dataset.

\begin{figure*}[t!]
    \centering

    \begin{subfigure}[t]{0.32\linewidth}
        \centering
        \includegraphics[width=\linewidth]{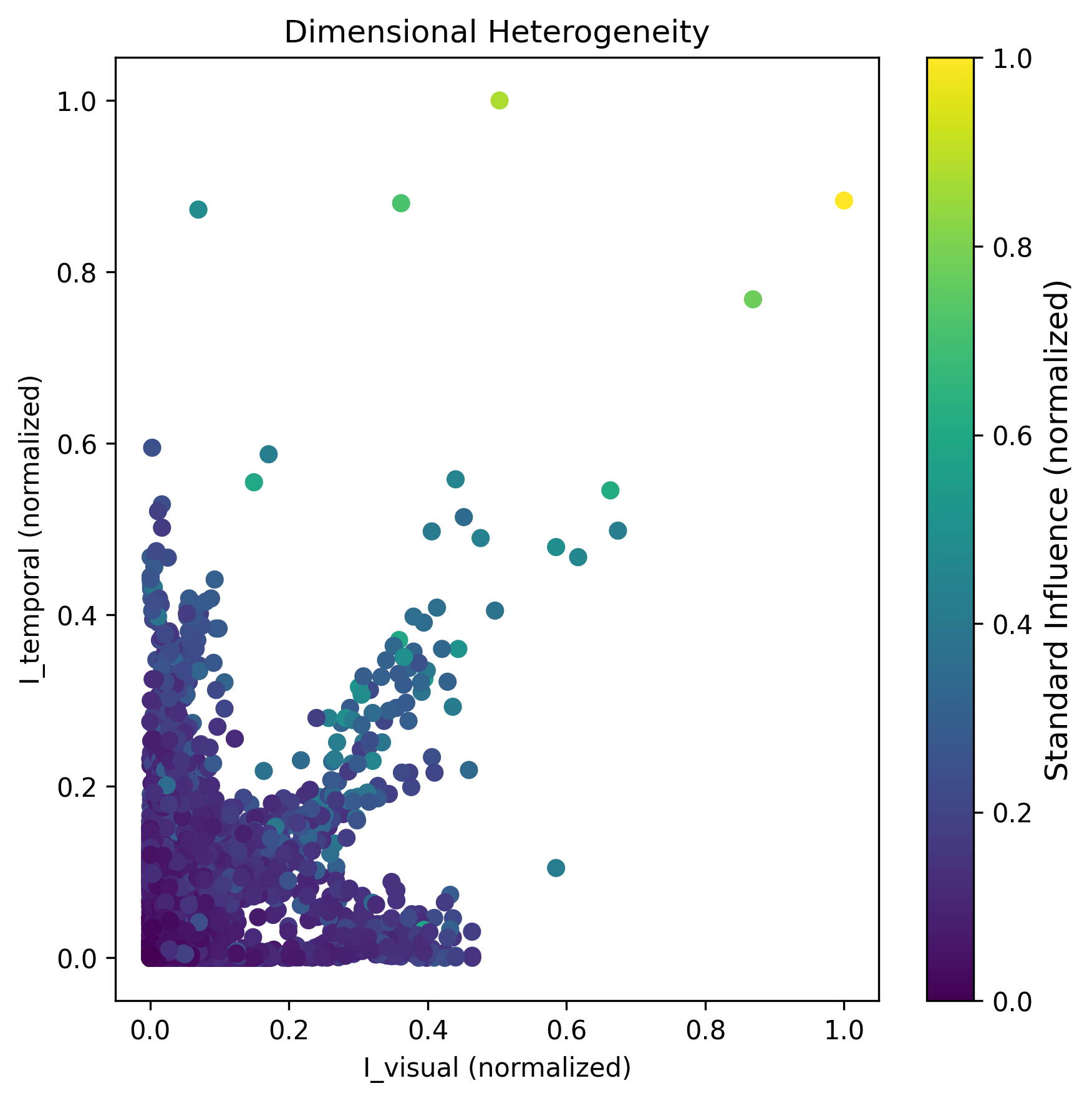}
        \caption{Visual vs Temporal}
    \end{subfigure}
    \hfill
    \begin{subfigure}[t]{0.32\linewidth}
        \centering
        \includegraphics[width=\linewidth]{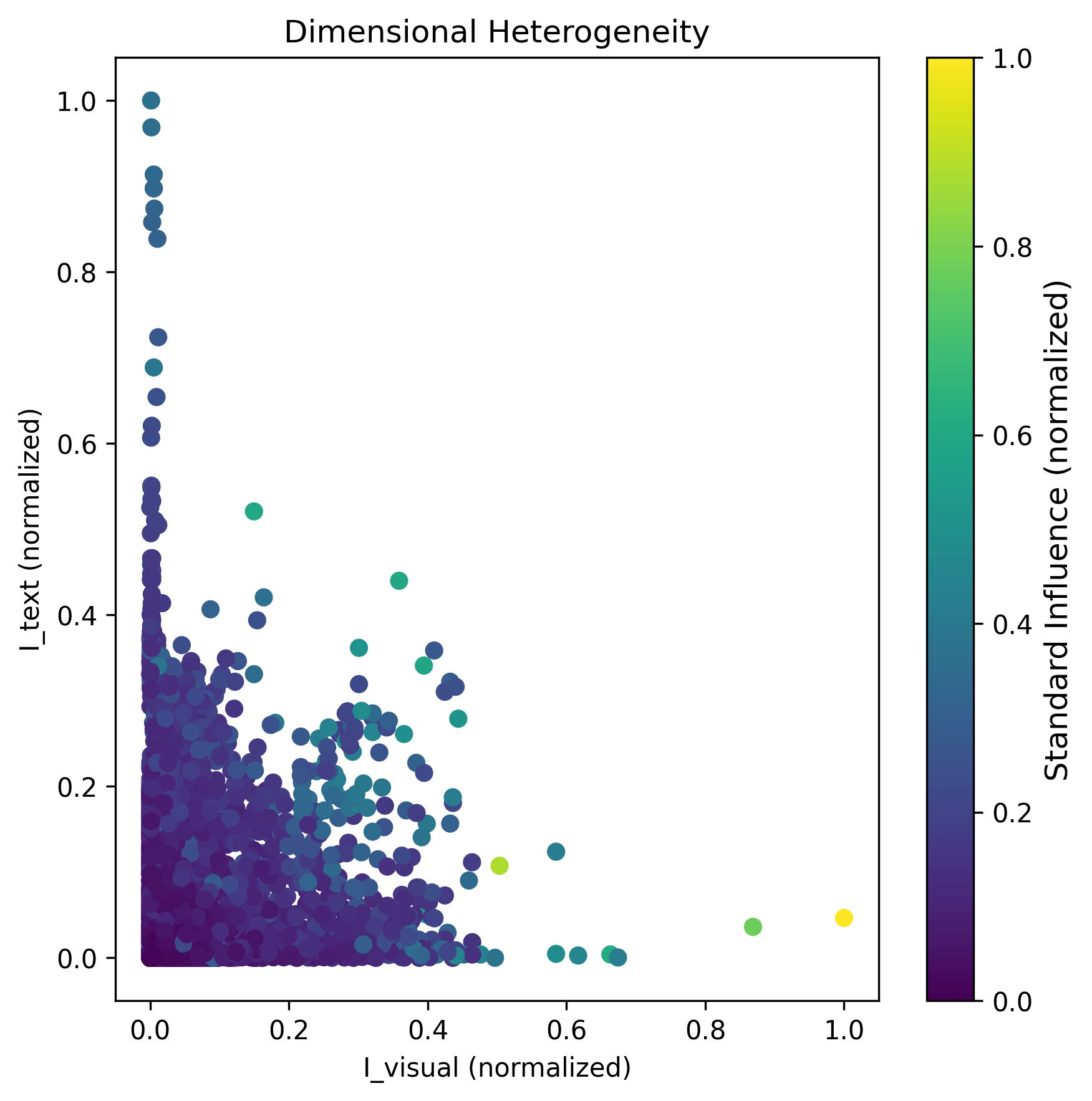}
        \caption{Visual vs Text Align}
    \end{subfigure}
    \hfill
    \begin{subfigure}[t]{0.32\linewidth}
        \centering
        \includegraphics[width=\linewidth]{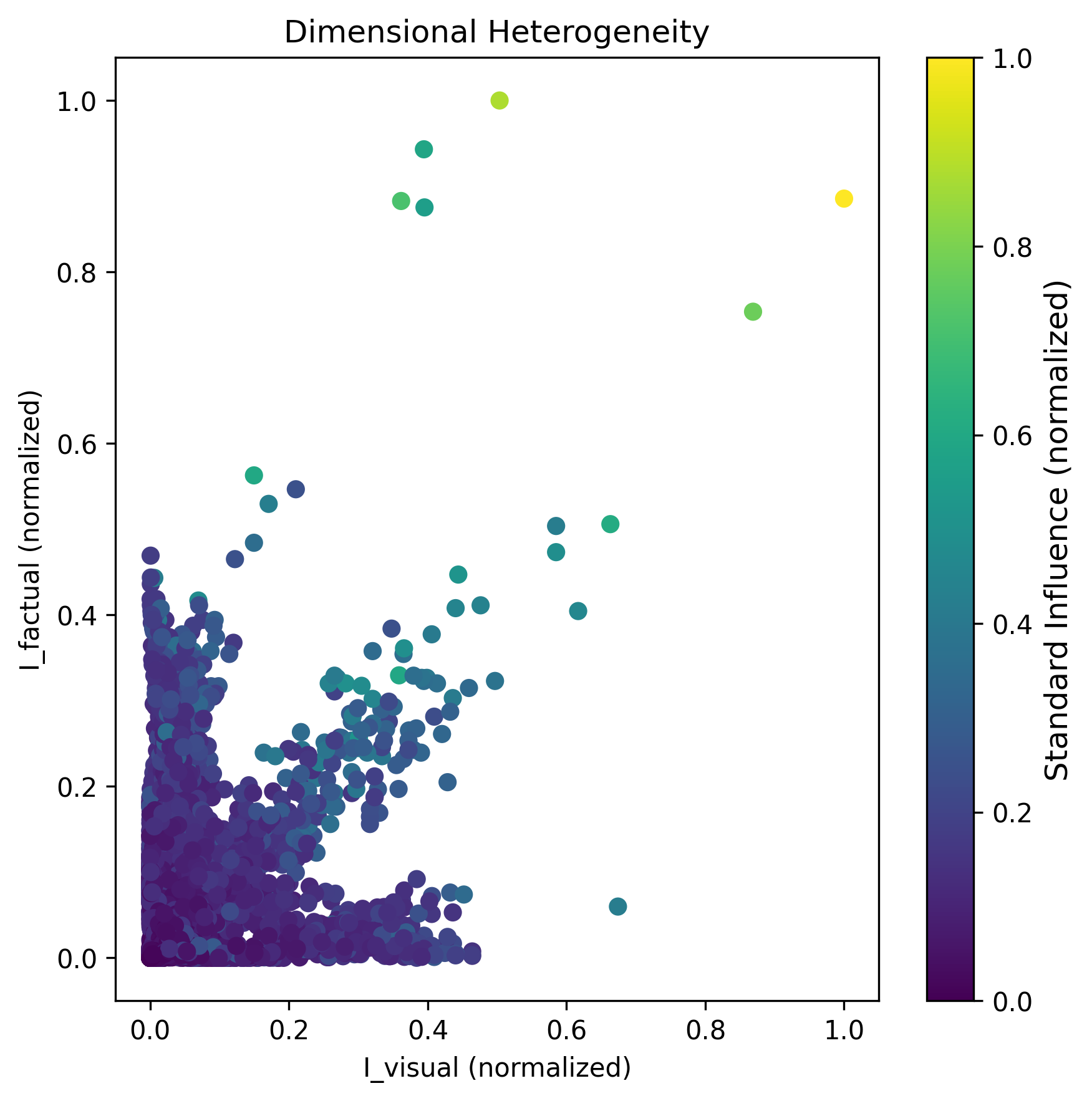}
        \caption{Visual vs Factual Consistency}
    \end{subfigure}

    \caption{\textbf{Dimensional Heterogeneity across reward dimensions.} 
    We visualize the relationship between \textit{Visual Quality} influence (x-axis) and three other reward dimensions (y-axis): (a) Temporal Consistency, (b) Text Alignment, and (c) Factual Consistency.
    The color indicates the normalized \textbf{Standard Influence} (Total Loss).
    The distinct ``L-shaped'' distribution reveals that supervision reliability is disentangled. 
    Crucially, many samples with high reward-disentangled influence (high values on y-axis) exhibit low Total Loss (represented by blue/purple colors), demonstrating that \textit{Standard Influence} fails to capture dimension-specific failure patterns (the masking effect).}
    \label{fig:visual_multi_scatter}
\end{figure*}

\subsection{Visualization of Dimensional Heterogeneity}
\label{sec:vis_analysis}

\noindent\textbf{Experimental Setup.} 
To validate \textit{Dimensional Heterogeneity}, we visualize pairwise distributions of normalized self-influence scores between \textit{Visual Quality} and other dimensions, including Temporal Consistency, Text-to-Video Alignment, and Factual Consistency, in Fig.~\ref{fig:visual_multi_scatter}. 
The color of each point denotes the standard influence computed from the aggregated total loss.

\textbf{Observation: The ``L-Shaped'' Distribution.}
In Fig.~\ref{fig:visual_multi_scatter}, the points form ``L-shaped'' patterns rather than diagonal clusters, indicating weak correlation between dimension-wise self-influence scores.
For example, a sample may have high self-influence on \textit{Text Alignment}, which we interpret as high supervision risk, while having low self-influence on \textit{Visual Quality}.
This suggests that supervision reliability is dimension-specific rather than uniformly reliable or unreliable across dimensions.

\textbf{Analysis: The Masking Effect.}
The color distribution further reveals the limitation of global aggregation.
Many points with low aggregated influence (blue) still lie in high-risk regions of individual dimensions, such as the high factual-consistency region in Fig.~\ref{fig:visual_multi_scatter}(c).
This indicates a \textbf{Masking Effect}: aggregated scalar influence can hide dimension-specific supervision risk by mixing signals across dimensions.
As a result, global pruning may retain samples with unreliable supervision in specific reward dimensions.
Full pairwise visualizations are provided in Appendix~\ref{appendix:fullpair}.

\subsection{Qualitative Analysis: Uncovering Label Noise in Semantic Alignment}
\label{sec:qualitative_analysis}

\noindent\textbf{Experimental Setup and Selection Criteria.}
To provide an intuitive understanding of Dimensional Heterogeneity, we focus our qualitative analysis on \textbf{Text-to-Video (T2V) Alignment}.
We specifically select samples located in the off-diagonal region of the influence distribution: those exhibiting extremely high T2V self-influence but low visual self-influence.
This intersection identifies samples where the visual generation is stable (low visual self-influence), but the semantic head significantly diverges from the ground-truth label (high semantic self-influence).

\begin{figure*}[th!]
    \centering
    \includegraphics[width=1\linewidth]{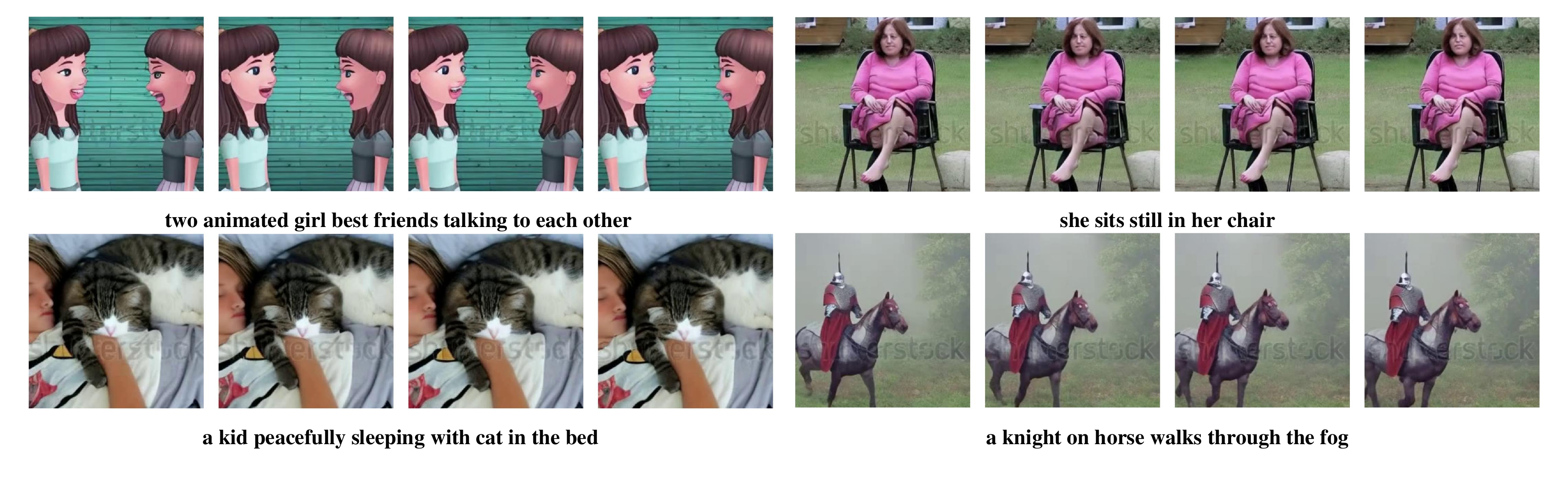} 
    \caption{\textbf{Qualitative Visualization of Label Noise in T2V Alignment.} 
    We show samples with \textbf{high T2V self-influence} but \textbf{low visual self-influence}. 
    Although they are labeled as \textbf{Score 1 (Worst Alignment)} in the training set, human re-evaluation rates them as Score 4, as the videos match substantial parts of their prompts, such as ``a knight on horse'' or ``kid sleeping with cat''. 
    This confirms label noise in the T2V alignment dimension for these selected samples.}
    \label{fig:qualitative_vis}
\end{figure*}

\textbf{Observation: Overly Strict Annotation Standards.}
Fig.~\ref{fig:qualitative_vis} shows representative samples labeled as Score 1 but re-evaluated by human inspectors as Score 4.
Among the top 1\% samples ranked by T2V self-influence, 84.8\% are assigned the lowest ground-truth label.
Manual inspection reveals a common pattern: these videos match most of the textual description but are still assigned the lowest alignment score.
This indicates that high T2V self-influence effectively prioritizes samples with label noise or overly strict annotations for human inspection.
Additional label-noise visualizations are provided in Appendix~\ref{app:vis_more}.

\subsection{Application: Human-in-the-Loop Label Correction}
\label{sec:label_correction}
\vspace{-1mm}
To demonstrate the practical utility of our Disentangled Self-Influence as a tool for prioritizing high-leverage label noise, we conduct a human-in-the-loop annotation refinement experiment.
We find that samples with exceptionally high self-influence scores often correspond to cases where the ground-truth label conflicts with the video content, as confirmed by human inspection.

\begin{wraptable}{r}{0.56\textwidth}
    \vspace{-10pt}
    \centering
    \caption{\textbf{Label correction results.} Spearman correlation ($\rho \times 100$) after checking the top 1\% T2V labels identified by our T2V self-influence score. The gray column highlights the improved Text alignment performance.}
    \label{tab:correction}

    \resizebox{0.54\textwidth}{!}{
    \begin{tabular}{l cccc | c}
        \toprule
        \textbf{Method} & \textbf{Vis.} & \textbf{Temp.} & \textbf{Dyn.} & \textbf{Fact.} & \textbf{Text} \\
        \midrule
        Baseline & 81.16 & 73.55 & 57.08 & 78.91 & \cellcolor{gray!15}50.61 \\
        \midrule
        Correction & \textbf{81.48} & 73.27 & 56.90 & \textbf{79.21} & \cellcolor{gray!15}\textbf{51.62} \\
        $\Delta$ & \textit{+0.32} & \textit{-0.28} & \textit{-0.18} & \textit{+0.30} & \cellcolor{gray!15}\textbf{\textit{+1.01}} \\
        \bottomrule
    \end{tabular}
    }
    \vspace{-10pt}
\end{wraptable}

\textbf{Experimental Setup.}
We sort the training data by the self-influence score of the \textit{T2V Alignment} dimension.
We select the top 1\% of samples with the highest T2V self-influence and manually re-examine their labels.
Consistent with our qualitative observation, human inspection shows that a significant portion of these samples are under-scored.
We then correct these labels according to human judgment.

\textbf{Results.}
Table~\ref{tab:correction} reports the performance comparison.
Retraining with our correction yields a improvement of \textbf{+1.01\%} in T2V Alignment correlation compared to the baseline.
Crucially, this gain is achieved without affecting the performance of other dimensions significantly, validating that our method accurately targets noise.

\section{Conclusion}
\label{sec:conclusion}
\vspace{-1mm}
In this paper, we identify \textit{Dimensional Heterogeneity} in MVRMs: supervision reliability can vary substantially across reward dimensions and may be masked by global scalar metrics.
To analyze this phenomenon, we introduce a disentangled influence framework that uses diagonal projection to estimate dimension-specific self-influence as a supervision-risk signal.
Based on this analysis, we develop two dataset refinement strategies, Dimension-Disentangled Pruning and Reweighting, for hard filtering and soft reweighting.
Experiments show that our approach improves reward-model alignment with human preferences over global filtering baselines.

Our qualitative and label-correction analyses further show that high self-influence can prioritize samples with annotation inconsistencies, often arising from overly rigid labeling heuristics and the open-world semantics captured by pre-trained VLMs.
These findings suggest that dimension-wise influence is not only useful for data refinement, but also for improving annotation standards in multidimensional video evaluation.
Overall, this work provides a practical data-centric paradigm for training more reliable MVRMs and supporting fine-grained evaluation of video generation.

\normalem
\bibliographystyle{unsrt}
\bibliography{icml2026}

@article{unterthiner2018towards,
  title={Towards accurate generative models of video: A new metric \& challenges},
  author={Unterthiner, Thomas and Van Steenkiste, Sjoerd and Kurach, Karol and Marinier, Raphael and Michalski, Marcin and Gelly, Sylvain},
  journal={arXiv preprint arXiv:1812.01717},
  year={2018}
}

@article{min2025understanding,
  title={Understanding Impact of Human Feedback via Influence Functions},
  author={Min, Taywon and Lee, Haeone and Kwon, Yongchan and Lee, Kimin},
  journal={arXiv preprint arXiv:2501.05790},
  year={2025}
}

@article{jiang2024genai,
  title={Genai arena: An open evaluation platform for generative models},
  author={Jiang, Dongfu and Ku, Max and Li, Tianle and Ni, Yuansheng and Sun, Shizhuo and Fan, Rongqi and Chen, Wenhu},
  journal={Advances in Neural Information Processing Systems},
  volume={37},
  pages={79889--79908},
  year={2024}
}

@article{wang2023modelscope,
  title={Modelscope text-to-video technical report},
  author={Wang, Jiuniu and Yuan, Hangjie and Chen, Dayou and Zhang, Yingya and Wang, Xiang and Zhang, Shiwei},
  journal={arXiv preprint arXiv:2308.06571},
  year={2023}
}

@article{kong2024hunyuanvideo,
  title={Hunyuanvideo: A systematic framework for large video generative models},
  author={Kong, Weijie and Tian, Qi and Zhang, Zijian and Min, Rox and Dai, Zuozhuo and Zhou, Jin and Xiong, Jiangfeng and Li, Xin and Wu, Bo and Zhang, Jianwei and others},
  journal={arXiv preprint arXiv:2412.03603},
  year={2024}
}

@article{team2025kling,
  title={Kling-Omni Technical Report},
  author={Team, Kling and Chen, Jialu and Ci, Yuanzheng and Du, Xiangyu and Feng, Zipeng and Gai, Kun and Guo, Sainan and Han, Feng and He, Jingbin and He, Kang and others},
  journal={arXiv preprint arXiv:2512.16776},
  year={2025}
}

@article{liu2023fetv,
  title={Fetv: A benchmark for fine-grained evaluation of open-domain text-to-video generation},
  author={Liu, Yuanxin and Li, Lei and Ren, Shuhuai and Gao, Rundong and Li, Shicheng and Chen, Sishuo and Sun, Xu and Hou, Lu},
  journal={Advances in Neural Information Processing Systems},
  volume={36},
  pages={62352--62387},
  year={2023}
}

@inproceedings{liu2024evalcrafter,
  title={Evalcrafter: Benchmarking and evaluating large video generation models},
  author={Liu, Yaofang and Cun, Xiaodong and Liu, Xuebo and Wang, Xintao and Zhang, Yong and Chen, Haoxin and Liu, Yang and Zeng, Tieyong and Chan, Raymond and Shan, Ying},
  booktitle={Proceedings of the IEEE/CVF Conference on Computer Vision and Pattern Recognition},
  pages={22139--22149},
  year={2024}
}

@inproceedings{barshan2020relatif,
  title={Relatif: Identifying explanatory training samples via relative influence},
  author={Barshan, Elnaz and Brunet, Marc-Etienne and Dziugaite, Gintare Karolina},
  booktitle={International Conference on Artificial Intelligence and Statistics},
  pages={1899--1909},
  year={2020},
  organization={PMLR}
}

@article{liu2025improving,
  title={Improving video generation with human feedback},
  author={Liu, Jie and Liu, Gongye and Liang, Jiajun and Yuan, Ziyang and Liu, Xiaokun and Zheng, Mingwu and Wu, Xiele and Wang, Qiulin and Xia, Menghan and Wang, Xintao and others},
  journal={arXiv preprint arXiv:2501.13918},
  year={2025}
}

@article{wu2024boosting,
  title={Boosting text-to-video generative model with mllms feedback},
  author={Wu, Xun and Huang, Shaohan and Wang, Guolong and Xiong, Jing and Wei, Furu},
  journal={Advances in Neural Information Processing Systems},
  volume={37},
  pages={139444--139469},
  year={2024}
}

@article{george2018fast,
  title={Fast approximate natural gradient descent in a kronecker factored eigenbasis},
  author={George, Thomas and Laurent, C{\'e}sar and Bouthillier, Xavier and Ballas, Nicolas and Vincent, Pascal},
  journal={Advances in neural information processing systems},
  volume={31},
  year={2018}
}

@article{ziegler2019fine,
  title={Fine-tuning language models from human preferences},
  author={Ziegler, Daniel M and Stiennon, Nisan and Wu, Jeffrey and Brown, Tom B and Radford, Alec and Amodei, Dario and Christiano, Paul and Irving, Geoffrey},
  journal={arXiv preprint arXiv:1909.08593},
  year={2019}
}

@article{linreasonable,
  title={Reasonable Effectiveness of Random Weighting: A Litmus Test for Multi-Task Learning},
  author={Lin, Baijiong and Ye, Feiyang and Zhang, Yu and Tsang, Ivor},
  journal={Transactions on Machine Learning Research}
}

@article{yu2020gradient,
  title={Gradient surgery for multi-task learning},
  author={Yu, Tianhe and Kumar, Saurabh and Gupta, Abhishek and Levine, Sergey and Hausman, Karol and Finn, Chelsea},
  journal={Advances in neural information processing systems},
  volume={33},
  pages={5824--5836},
  year={2020}
}

@article{hong2022cogvideo,
  title={Cogvideo: Large-scale pretraining for text-to-video generation via transformers},
  author={Hong, Wenyi and Ding, Ming and Zheng, Wendi and Liu, Xinghan and Tang, Jie},
  journal={arXiv preprint arXiv:2205.15868},
  year={2022}
}

@article{blattmann2023stable,
  title={Stable video diffusion: Scaling latent video diffusion models to large datasets},
  author={Blattmann, Andreas and Dockhorn, Tim and Kulal, Sumith and Mendelevitch, Daniel and Kilian, Maciej and Lorenz, Dominik and Levi, Yam and English, Zion and Voleti, Vikram and Letts, Adam and others},
  journal={arXiv preprint arXiv:2311.15127},
  year={2023}
}

@article{wan2025wan,
  title={Wan: Open and advanced large-scale video generative models},
  author={Wan, Team and Wang, Ang and Ai, Baole and Wen, Bin and Mao, Chaojie and Xie, Chen-Wei and Chen, Di and Yu, Feiwu and Zhao, Haiming and Yang, Jianxiao and others},
  journal={arXiv preprint arXiv:2503.20314},
  year={2025}
}

@inproceedings{kendall2018multi,
  title={Multi-task learning using uncertainty to weigh losses for scene geometry and semantics},
  author={Kendall, Alex and Gal, Yarin and Cipolla, Roberto},
  booktitle={Proceedings of the IEEE conference on computer vision and pattern recognition},
  pages={7482--7491},
  year={2018}
}

@article{harrison2024variational,
  title={Variational Bayesian last layers},
  author={Harrison, James and Willes, John and Snoek, Jasper},
  journal={arXiv preprint arXiv:2404.11599},
  year={2024}
}

@article{christiano2017deep,
  title={Deep reinforcement learning from human preferences},
  author={Christiano, Paul F and Leike, Jan and Brown, Tom and Martic, Miljan and Legg, Shane and Amodei, Dario},
  journal={Advances in neural information processing systems},
  volume={30},
  year={2017}
}

@article{stiennon2020learning,
  title={Learning to summarize with human feedback},
  author={Stiennon, Nisan and Ouyang, Long and Wu, Jeffrey and Ziegler, Daniel and Lowe, Ryan and Voss, Chelsea and Radford, Alec and Amodei, Dario and Christiano, Paul F},
  journal={Advances in neural information processing systems},
  volume={33},
  pages={3008--3021},
  year={2020}
}

@inproceedings{tongmj,
  title={MJ-Video: Benchmarking and Rewarding Video Generation with Fine-Grained Video Preference},
  author={Tong, Haibo and Wang, Zhaoyang and Chen, Zhaorun and Ji, Haonian and Qiu, Shi and Han, Siwei and Geng, Kexin and Xue, Zhongkai and Zhou, Yiyang and Xia, Peng and others},
  booktitle={The Thirty-ninth Annual Conference on Neural Information Processing Systems}
}

@article{xu2024visionreward,
  title={Visionreward: Fine-grained multi-dimensional human preference learning for image and video generation},
  author={Xu, Jiazheng and Huang, Yu and Cheng, Jiale and Yang, Yuanming and Xu, Jiajun and Wang, Yuan and Duan, Wenbo and Yang, Shen and Jin, Qunlin and Li, Shurun and others},
  journal={arXiv preprint arXiv:2412.21059},
  year={2024}
}

@inproceedings{lin2024video,
  title={Video-llava: Learning united visual representation by alignment before projection},
  author={Lin, Bin and Ye, Yang and Zhu, Bin and Cui, Jiaxi and Ning, Munan and Jin, Peng and Yuan, Li},
  booktitle={Proceedings of the 2024 conference on empirical methods in natural language processing},
  pages={5971--5984},
  year={2024}
}

@inproceedings{hessel2021clipscore,
  title={Clipscore: A reference-free evaluation metric for image captioning},
  author={Hessel, Jack and Holtzman, Ari and Forbes, Maxwell and Le Bras, Ronan and Choi, Yejin},
  booktitle={Proceedings of the 2021 conference on empirical methods in natural language processing},
  pages={7514--7528},
  year={2021}
}

@inproceedings{huang2024vbench,
  title={Vbench: Comprehensive benchmark suite for video generative models},
  author={Huang, Ziqi and He, Yinan and Yu, Jiashuo and Zhang, Fan and Si, Chenyang and Jiang, Yuming and Zhang, Yuanhan and Wu, Tianxing and Jin, Qingyang and Chanpaisit, Nattapol and others},
  booktitle={Proceedings of the IEEE/CVF Conference on Computer Vision and Pattern Recognition},
  pages={21807--21818},
  year={2024}
}

@article{jiangmantis,
  title={Mantis: Interleaved Multi-Image Instruction Tuning},
  author={Jiang, Dongfu and He, Xuan and Zeng, Huaye and Wei, Cong and Ku, Max and Liu, Qian and Chen, Wenhu},
  journal={Transactions on Machine Learning Research}
}

@inproceedings{he2024videoscore,
  title={Videoscore: Building automatic metrics to simulate fine-grained human feedback for video generation},
  author={He, Xuan and Jiang, Dongfu and Zhang, Ge and Ku, Max and Soni, Achint and Siu, Sherman and Chen, Haonan and Chandra, Abhranil and Jiang, Ziyan and Arulraj, Aaran and others},
  booktitle={Proceedings of the 2024 Conference on Empirical Methods in Natural Language Processing},
  pages={2105--2123},
  year={2024}
}

@article{park2023trak,
  title={Trak: Attributing model behavior at scale},
  author={Park, Sung Min and Georgiev, Kristian and Ilyas, Andrew and Leclerc, Guillaume and Madry, Aleksander},
  journal={arXiv preprint arXiv:2303.14186},
  year={2023}
}

@article{tan2024data,
  title={Data pruning via moving-one-sample-out},
  author={Tan, Haoru and Wu, Sitong and Du, Fei and Chen, Yukang and Wang, Zhibin and Wang, Fan and Qi, Xiaojuan},
  journal={Advances in Neural Information Processing Systems},
  volume={36},
  year={2024}
}

@article{pruthi2020estimating,
  title={Estimating training data influence by tracing gradient descent},
  author={Pruthi, Garima and Liu, Frederick and Kale, Satyen and Sundararajan, Mukund},
  journal={Advances in Neural Information Processing Systems},
  volume={33},
  pages={19920--19930},
  year={2020}
}

@inproceedings{koh2017understanding,
  title={Understanding black-box predictions via influence functions},
  author={Koh, Pang Wei and Liang, Percy},
  booktitle={International conference on machine learning},
  pages={1885--1894},
  year={2017},
  organization={PMLR}
}

\newpage
\appendix
\onecolumn

\section{Proof of Theorem \ref{thm:disentangled_influence}}
\label{app:proof_thm1}

In this section, we provide the detailed derivation of the Disentangled Influence Decomposition presented in Theorem \ref{thm:disentangled_influence}.

\subsection{Preliminaries: Influence Functions}

First, we briefly recall the standard derivation of influence functions~\citep{koh2017understanding} to establish the relationship between data upweighting and parameter updates.

Let $\mathcal{R}(\boldsymbol{\theta}) = \frac{1}{N} \sum_{i=1}^N \mathcal{L}(z_i; \boldsymbol{\theta})$ be the empirical risk. Let $\hat{\boldsymbol{\theta}}$ be the empirical risk minimizer, satisfying the first-order optimality condition $\nabla_{\boldsymbol{\theta}} \mathcal{R}(\hat{\boldsymbol{\theta}}) = 0$.
Consider upweighting a specific training sample $z$ by a small scalar $\epsilon$. The perturbed objective is given by:
\begin{equation}
    \mathcal{R}_{\epsilon, z}(\boldsymbol{\theta}) = \mathcal{R}(\boldsymbol{\theta}) + \epsilon \mathcal{L}(z; \boldsymbol{\theta}).
\end{equation}
Let $\hat{\boldsymbol{\theta}}_{\epsilon, z}$ be the minimizer of this perturbed objective. The optimality condition requires:
\begin{equation} \label{eq:optimality}
    \nabla_{\boldsymbol{\theta}} \mathcal{R}_{\epsilon, z}(\hat{\boldsymbol{\theta}}_{\epsilon, z}) = \nabla_{\boldsymbol{\theta}} \mathcal{R}(\hat{\boldsymbol{\theta}}_{\epsilon, z}) + \epsilon \nabla_{\boldsymbol{\theta}} \mathcal{L}(z; \hat{\boldsymbol{\theta}}_{\epsilon, z}) = 0.
\end{equation}
We perform a first-order Taylor expansion of $\nabla_{\boldsymbol{\theta}} \mathcal{R}(\cdot)$ around the original optimal parameter $\hat{\boldsymbol{\theta}}$:
\begin{equation}
    \nabla_{\boldsymbol{\theta}} \mathcal{R}(\hat{\boldsymbol{\theta}}_{\epsilon, z}) \approx \nabla_{\boldsymbol{\theta}} \mathcal{R}(\hat{\boldsymbol{\theta}}) + \mathbf{H}_{\hat{\boldsymbol{\theta}}} (\hat{\boldsymbol{\theta}}_{\epsilon, z} - \hat{\boldsymbol{\theta}}),
\end{equation}
where $\mathbf{H}_{\hat{\boldsymbol{\theta}}} = \nabla^2_{\boldsymbol{\theta}} \mathcal{R}(\hat{\boldsymbol{\theta}})$ is the Hessian matrix. Since $\nabla_{\boldsymbol{\theta}} \mathcal{R}(\hat{\boldsymbol{\theta}}) = 0$, substituting the expansion back into Eq.~\eqref{eq:optimality} and keeping terms to order $O(\epsilon)$ yields:
\begin{equation}
    \mathbf{H}_{\hat{\boldsymbol{\theta}}} (\hat{\boldsymbol{\theta}}_{\epsilon, z} - \hat{\boldsymbol{\theta}}) + \epsilon \nabla_{\boldsymbol{\theta}} \mathcal{L}(z; \hat{\boldsymbol{\theta}}) \approx 0.
\end{equation}
Solving for the parameter change rate $\frac{d\hat{\boldsymbol{\theta}}}{d\epsilon}$:
\begin{equation} \label{eq:param_change}
    \left. \frac{d\hat{\boldsymbol{\theta}}_{\epsilon, z}}{d\epsilon} \right|_{\epsilon=0} = -\mathbf{H}_{\hat{\boldsymbol{\theta}}}^{-1} \nabla_{\boldsymbol{\theta}} \mathcal{L}(z; \hat{\boldsymbol{\theta}}).
\end{equation}
Finally, applying the chain rule to the test loss $\mathcal{L}(z_{test}; \hat{\boldsymbol{\theta}})$, we obtain the influence on the test sample:
\begin{equation} \label{eq:influence_def}
    \mathcal{I}_{up}(z, z_{test}) := \left. \frac{d\mathcal{L}(z_{test}; \hat{\boldsymbol{\theta}}_{\epsilon, z})}{d\epsilon} \right|_{\epsilon=0} = \nabla_{\boldsymbol{\theta}} \mathcal{L}(z_{test}; \hat{\boldsymbol{\theta}})^\top \left( -\mathbf{H}_{\hat{\boldsymbol{\theta}}}^{-1} \nabla_{\boldsymbol{\theta}} \mathcal{L}(z; \hat{\boldsymbol{\theta}}) \right).
\end{equation}
For simplicity in subsequent derivations, we omit the negative sign as we focus on the magnitude and direction of interaction, following standard convention in influence analysis logic.

\subsection{Derivation of Disentangled Influence Function}

We now proceed to prove the decomposition claim. In a multidimensional reward model, the total loss function is an aggregation of $K$ distinct objective losses:
\begin{equation}
    \mathcal{L}(x; \boldsymbol{\theta}) = \sum_{k=1}^K \lambda_k \mathcal{L}_k(x; \boldsymbol{\theta}),
\end{equation}
where $\lambda_k$ are fixed weighting coefficients.

Due to the \textbf{linearity of the gradient operator}, the gradient of the total loss with respect to parameters $\boldsymbol{\theta}$ is the weighted sum of the gradients of individual objective losses.
For the test sample $z_{test}$, we have:
\begin{equation} \label{eq:grad_test}
    \nabla_{\boldsymbol{\theta}} \mathcal{L}(z_{test}; \hat{\boldsymbol{\theta}}) = \sum_{j=1}^K \lambda_j \nabla_{\boldsymbol{\theta}} \mathcal{L}_j(z_{test}; \hat{\boldsymbol{\theta}}).
\end{equation}
Similarly, for the training sample $z$:
\begin{equation} \label{eq:grad_train}
    \nabla_{\boldsymbol{\theta}} \mathcal{L}(z; \hat{\boldsymbol{\theta}}) = \sum_{k=1}^K \lambda_k \nabla_{\boldsymbol{\theta}} \mathcal{L}_k(z; \hat{\boldsymbol{\theta}}).
\end{equation}

Substituting Eq.~\eqref{eq:grad_test} and Eq.~\eqref{eq:grad_train} into the standard influence definition (Eq.~\eqref{eq:influence_def}):

\begin{equation}
\begin{split}
    \mathcal{I}_{up}(z, z_{test}) &= \left( \sum_{j=1}^K \lambda_j \nabla_{\boldsymbol{\theta}} \mathcal{L}_j(z_{test}; \hat{\boldsymbol{\theta}}) \right)^\top \mathbf{H}_{\hat{\boldsymbol{\theta}}}^{-1} \left( \sum_{k=1}^K \lambda_k \nabla_{\boldsymbol{\theta}} \mathcal{L}_k(z; \hat{\boldsymbol{\theta}}) \right) \\
    &= \sum_{j=1}^K \left[ \left( \lambda_j \nabla_{\boldsymbol{\theta}} \mathcal{L}_j(z_{test}; \hat{\boldsymbol{\theta}}) \right)^\top \mathbf{H}_{\hat{\boldsymbol{\theta}}}^{-1} \left( \sum_{k=1}^K \lambda_k \nabla_{\boldsymbol{\theta}} \mathcal{L}_k(z; \hat{\boldsymbol{\theta}}) \right) \right] \quad \text{(Distributivity over addition)} \\
    &= \sum_{j=1}^K \sum_{k=1}^K \left[ \left( \lambda_j \nabla_{\boldsymbol{\theta}} \mathcal{L}_j(z_{test}; \hat{\boldsymbol{\theta}}) \right)^\top \mathbf{H}_{\hat{\boldsymbol{\theta}}}^{-1} \left( \lambda_k \nabla_{\boldsymbol{\theta}} \mathcal{L}_k(z; \hat{\boldsymbol{\theta}}) \right) \right] \\
    &= \sum_{j=1}^K \sum_{k=1}^K \underbrace{ \lambda_j \lambda_k \nabla_{\boldsymbol{\theta}} \mathcal{L}_j(z_{test}; \hat{\boldsymbol{\theta}})^\top \mathbf{H}_{\hat{\boldsymbol{\theta}}}^{-1} \nabla_{\boldsymbol{\theta}} \mathcal{L}_k(z; \hat{\boldsymbol{\theta}}) }_{\phi_{j,k}}.
\end{split}
\end{equation}

Here, each term $\phi_{j,k}$ corresponds to the influence propagated from the $k$-th dimension of the training sample to the $j$-th dimension of the test sample via the shared parameter space (specifically, the Inverse Hessian).
Thus, the total influence is exactly the sum of all elements in the $K \times K$ interaction matrix $\boldsymbol{\mathcal{M}}(z, z_{test})$ where $\boldsymbol{\mathcal{M}}_{j,k} = \phi_{j,k}$.

This completes the proof. \hfill $\square$

\subsection{Efficiency Analysis: Closed-Form Derivation and Complexity}
\label{sec:efficiency}

A fundamental limitation of standard influence functions is the computational prohibitive cost of Hessian inversion and per-sample gradient computation over the entire parameter space. 
We circumvent this bottleneck by restricting the analysis to the last linear projection layer and deriving a closed-form solution for the Mean Squared Error (MSE) loss, which eliminates the need for backward propagation entirely.

\paragraph{Closed-Form Derivation under MSE.}
Let $f_{\boldsymbol{\theta}}(z)$ denote the reward model. We focus on the influence relative to the parameters of the final regression head.
Let $\mathbf{h}(z) \in \mathbb{R}^d$ be the feature representation (embedding) of sample $z$ produced by the backbone, and let $\mathbf{w}_k \in \mathbb{R}^d$ be the weight vector for the $k$-th evaluation dimension in the last linear layer. 
The predicted score $\hat{y}_k$ and the ground truth label $y_k$ define the dimension-specific MSE loss:
\begin{equation}
    \mathcal{L}_k(z; \mathbf{w}_k) = \frac{1}{2} \left( \hat{y}_k - y_k \right)^2 = \frac{1}{2} \left( \mathbf{w}_k^\top \mathbf{h}(z) - y_k \right)^2.
\end{equation}
Using the chain rule, the gradient of the loss with respect to the last-layer weights $\mathbf{w}_k$ is:
\begin{equation}
    \nabla_{\mathbf{w}_k} \mathcal{L}_k(z) = \left( \mathbf{w}_k^\top \mathbf{h}(z) - y_k \right) \cdot \frac{\partial (\mathbf{w}_k^\top \mathbf{h}(z))}{\partial \mathbf{w}_k} = r_k(z) \cdot \mathbf{h}(z),
\end{equation}
where $r_k(z) = \hat{y}_k - y_k$ is the scalar residual (prediction error).
Substituting this into the Self-Influence definition (Eq.~\ref{eq:self_influence}) and employing the Identity Hessian approximation ($\mathbf{H} \approx \mathbf{I}$), we obtain an explicit, computationally efficient expression:
\begin{equation} \label{eq:closed_form_influence}
    \mathcal{S}_k(z) \approx \left\| \nabla_{\mathbf{w}_k} \mathcal{L}_k(z) \right\|_2^2 = \left\| r_k(z) \cdot \mathbf{h}(z) \right\|_2^2 = r_k(z)^2 \cdot \|\mathbf{h}(z)\|_2^2.
\end{equation}
Eq.~\eqref{eq:closed_form_influence} reveals that the influence score is simply the product of the squared residual and the squared L2-norm of the feature embedding. 
Crucially, both terms are available immediately after the forward pass, bypassing the gradient backpropagation process.

\paragraph{Complexity Analysis.}
We strictly compare the computational complexity of our method against standard TracIn calculation.
Let $N$ be the number of samples, $P$ be the total number of model parameters (e.g., $P \approx 8 \times 10^9$ for VLM-based rewards), and $d$ be the hidden dimension of the last layer (e.g., $d = 4096$).

\begin{itemize}
    \item \textbf{Standard Influence (Full-Parameter):} Computing per-sample gradients $\nabla_{\boldsymbol{\theta}} \mathcal{L}(z_i)$ necessitates a full backward propagation through the deep backbone. 
    Even with efficient approximations, the computational cost is proportional to the total parameter count $P$.
    $$ \mathcal{T}_{\text{Standard}} \propto \mathcal{O}(N \cdot P_{backward}) $$
    
    \item \textbf{Ours (Last-Layer Closed-Form):} Our method relies solely on forward-pass statistics. The calculation of Eq.~\eqref{eq:closed_form_influence} involves only vector operations in the low-dimensional parameter space ($d \ll P$).
    $$ \mathcal{T}_{\text{Ours}} \propto \mathcal{O}(N \cdot d_{vector}) $$
\end{itemize}

\textbf{Conclusion: Dual-Level Acceleration.}
Our approach achieves computational efficiency through two synergistic mechanisms:
\begin{enumerate}
    \item \textbf{Parameter-Space Reduction:} By restricting the influence analysis to the last linear layer, we reduce the problem scale from the vast full-parameter space ($P \approx 8B$) to the low-dimensional feature space ($d \approx 4096$), realizing a complexity reduction of $\mathcal{O}(P) \to \mathcal{O}(d)$.
    \item \textbf{Backward-Free Estimation:} Crucially, our closed-form derivation allows us to compute the exact gradient norm using only forward-pass statistics (residuals and embeddings). This completely eliminates the need for the computationally expensive backpropagation process (backward pass), which typically consumes significantly more time and memory than the forward pass.
\end{enumerate}
Together, these structural optimizations ensure that influence profiling incurs negligible overhead relative to standard inference, guaranteeing scalability for large-scale video datasets.

\subsection{Algorithm}

\begin{algorithm}[h]
\caption{Dimension-Disentangled Pruning (DDP)}
\label{alg:rdap}
\begin{algorithmic}[1]
\REQUIRE Training dataset $\mathcal{D} = \{z_i\}_{i=1}^N$; Probe model parameters $\hat{\boldsymbol{\theta}}$; Loss functions $\{\mathcal{L}_k\}_{k=1}^K$; Per-dimension pruning ratio $\rho$.
\ENSURE Clean Dataset $\mathcal{D}_{clean}$.

\STATE Initialize high-risk supervision index set $\mathcal{I}_{drop} \leftarrow \emptyset$
\STATE Initialize score matrix $\mathbf{S} \leftarrow \mathbf{0}_{N \times K}$

\STATE \textit{// Phase 1: Influence Profiling}
\FOR{$i = 1$ to $N$}
    \FOR{$k = 1$ to $K$}
        \STATE Compute Self-Influence: $\mathbf{S}_{i,k} \leftarrow \nabla_{\boldsymbol{\theta}} \mathcal{L}_k(z_i; \hat{\boldsymbol{\theta}})^\top  \nabla_{\boldsymbol{\theta}} \mathcal{L}_k(z_i; \hat{\boldsymbol{\theta}})$
    \ENDFOR
\ENDFOR

\STATE \textit{// Phase 2: Union of Dimension-wise high-risk supervision data}
\FOR{$k = 1$ to $K$}
    \STATE Determine threshold $\tau_k$ as the $(1-\rho)$-quantile of column $\mathbf{S}_{:,k}$
    \FOR{$i = 1$ to $N$}
        \IF{$\mathbf{S}_{i,k} > \tau_k$}
            \STATE Add index $i$ to $\mathcal{I}_{drop}$ \COMMENT{Mark sample for removal if \textit{any} dimension fails}
        \ENDIF
    \ENDFOR
\ENDFOR

\STATE \textit{// Phase 3: Dataset Construction}
\STATE $\mathcal{D}_{clean} \leftarrow \{ z_i \mid i \notin \mathcal{I}_{drop} \}$

\STATE \textbf{return} $\mathcal{D}_{clean}$
\end{algorithmic}
\end{algorithm}

\begin{algorithm}[h]
\caption{Dimension-Disentangled Reweighting (DDR)}
\label{alg:rdar}
\begin{algorithmic}[1]
\REQUIRE Training dataset $\mathcal{D} = \{z_i\}_{i=1}^N$; Self-Influence Matrix $\mathbf{S} \in \mathbb{R}^{N \times K}$; Temperature $\tau$.
\ENSURE Dimension-specific Weight Matrix $\mathbf{W} \in \mathbb{R}^{N \times K}$.

\STATE Initialize weight matrix $\mathbf{W} \leftarrow \mathbf{0}_{N \times K}$

\STATE \textit{// Phase 1: Column-wise Global Statistics}
\FOR{$k = 1$ to $K$}
    \STATE Compute mean $\mu_k$ and std $\sigma_k$ of influence column $\mathbf{S}_{:, k}$
    \FOR{$i = 1$ to $N$}
        \STATE $\hat{\mathcal{S}}_{i,k} \leftarrow (\mathbf{S}_{i,k} - \mu_k) / (\sigma_k + \epsilon)$ \COMMENT{Z-score normalization}
    \ENDFOR
\ENDFOR

\STATE \textit{// Phase 2: Independent Mapping}
\FOR{$i = 1$ to $N$}
    \FOR{$k = 1$ to $K$}
        \STATE $w_{i,k} \leftarrow \left( 1 + \exp( \hat{\mathcal{S}}_{i,k} / \tau ) \right)^{-1}$ \COMMENT{sigmoid}
    \ENDFOR
\ENDFOR

\STATE \textit{// Phase 3: Global Mean Scaling}
\STATE Calculate global mean factor $\bar{w} = \frac{1}{N \cdot K} \sum_{i,k} w_{i,k}$
\STATE $\mathbf{W} \leftarrow \mathbf{W} / \bar{w}$ \COMMENT{Maintain effective learning rate}

\STATE \textbf{return} $\mathbf{W}$
\end{algorithmic}
\end{algorithm}

\section{Additional Analysis}
\subsection{Training Details}
\label{traindetails}
We utilize the VideoFeedback dataset~\citep{he2024videoscore}, comprising approximately 36,900 text-video pairs with human annotations across 5 dimensions.
All models are fully finetuned for 1 epoch using the AdamW optimizer with a learning rate of 1e-5 on Nvidia B200 using 16GPU hours.
Unless otherwise stated, we use the Spearman correlation ($\rho$) on the \textit{VideoFeedback-test} set~\citep{he2024videoscore} as the primary evaluation metrics.

\subsection{Model Generalization}
\label{sec:generalization}

To verify the universality of our proposed framework, we extend our experiments to different MVRM architectures.
Table~\ref{tab:generalization} presents the performance on Identis 8B~\citep{jiangmantis}.
We observe that both DDP and DDR consistently outperform the standard training baseline across all architectures.
This indicates that the \textit{Dimensional Heterogeneity} and \textit{data reliability} we identified are intrinsic data properties rather than model-specific artifacts.

\begin{table}[h]
    \centering
    \caption{\textbf{Generalization Analysis.} We report the Spearman correlation (\%) of our methods applied to the \textbf{Idefics2-8B} architecture. Both Dimension-Disentangled Pruning (DDP) and Reweighting (DDR) consistently outperform the baseline across all five evaluation dimensions.}
    \label{tab:generalization}
    \resizebox{1.0\linewidth}{!}{
        \begin{tabular}{l|c|ccccc}
            \toprule
            \multirow{2}{*}{\textbf{Model Architecture}} & \multirow{2}{*}{\textbf{Method}} & \multicolumn{5}{c}{\textbf{Evaluation Dimensions (Spearman $\rho$ \%)}} \\
            \cmidrule(lr){3-7}
             & & \textbf{Visual} & \textbf{Text} & \textbf{Temporal} & \textbf{Factual} & \textbf{Motion} \\ 
            \midrule
            
            \multirow{3}{*}{\textbf{Idefics2-8B}} 
            & Baseline & 79.36 & 72.34 & 53.23 & 48.63 & 76.89 \\
            \cmidrule{2-7}
            & \textbf{DDP} (Pruning) & 79.94 & 73.27 & \textbf{58.92} & 49.54 & 78.60 \\
            & \textbf{DDR} (Reweighting) & \textbf{81.16} & \textbf{73.60} & 54.88 & \textbf{50.44} & \textbf{79.83} \\
            
            \bottomrule
        \end{tabular}
    }
\end{table}

\subsection{Full Pairwise Visualization}
\label{appendix:fullpair}

To provide a comprehensive view of the relationships between dimension-specific self-influence scores, we present the full pairwise visualization of all evaluation dimensions in Fig.~\ref{fig:appendix_scatter}. Each subplot corresponds to a pair of dimensions $(d_i, d_j)$, where each point represents a training sample, plotted according to its normalized self-influence score along the two selected dimensions.

The visualization is constructed using the same MVRM and training set as in the main experiments. For each sample, we compute the diagonal entries of the influence matrix using the last-layer TracIn approximation, and apply min-max normalization within each dimension to ensure comparability across scales. 

As shown in Fig.~\ref{fig:appendix_scatter}, the majority of samples exhibit weak correlation across different dimensions, while a non-negligible subset forms pronounced high-risk supervision data that are highly influential in one dimension but marginal in others. This observation supports the masking effect discussed in Section~\ref{sec:vis_analysis}, where aggregated scalar influence criteria fail to capture dimension-specific data quality.

\begin{figure*}[t]
\centering
\subfloat[Dynamic vs. Factual]{
    \includegraphics[width=0.32\textwidth]{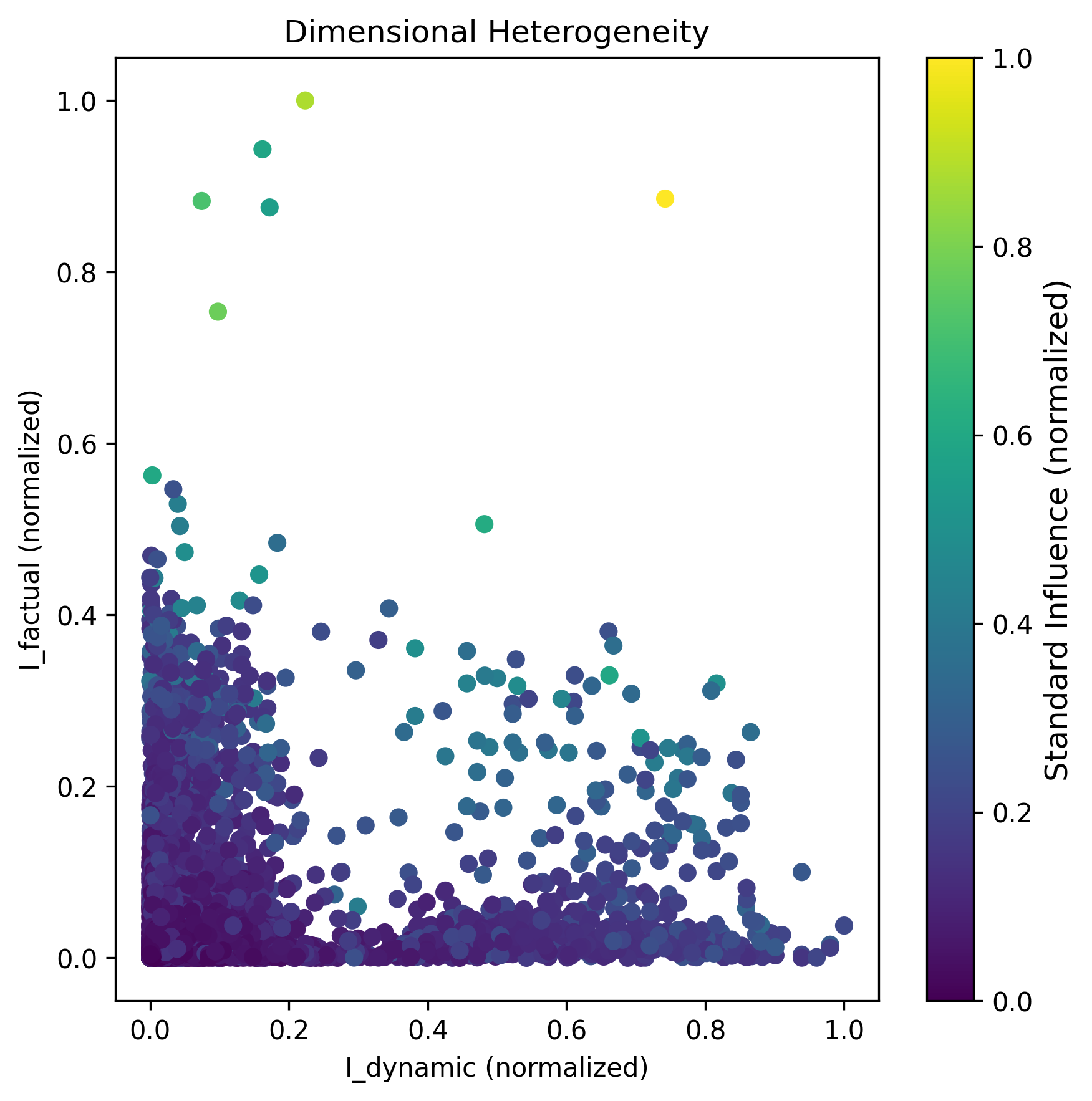}
}
\hfill
\subfloat[Dynamic vs. Text]{
    \includegraphics[width=0.32\textwidth]{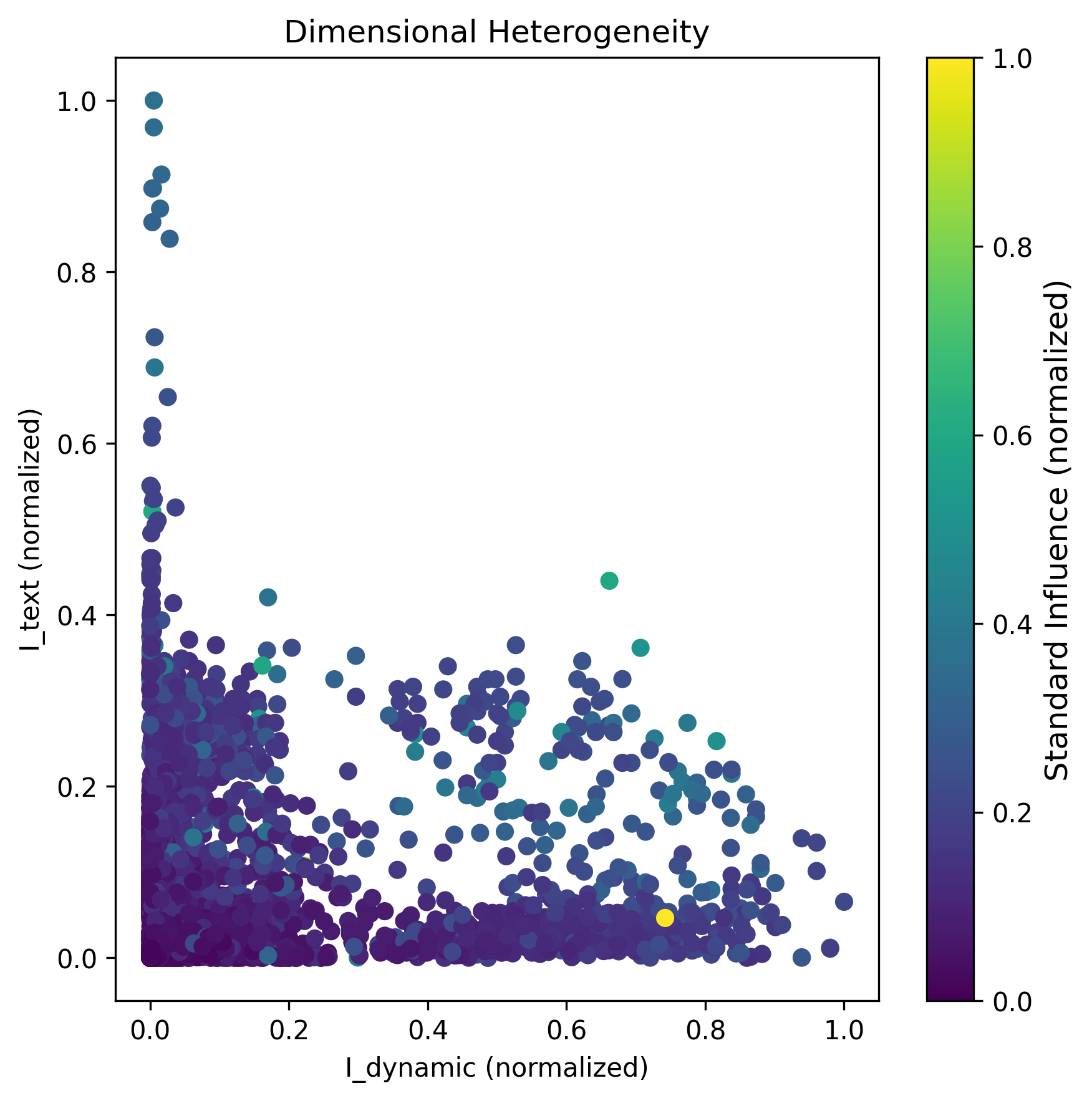}
}
\hfill
\subfloat[Temporal vs. Dynamic]{
    \includegraphics[width=0.32\textwidth]{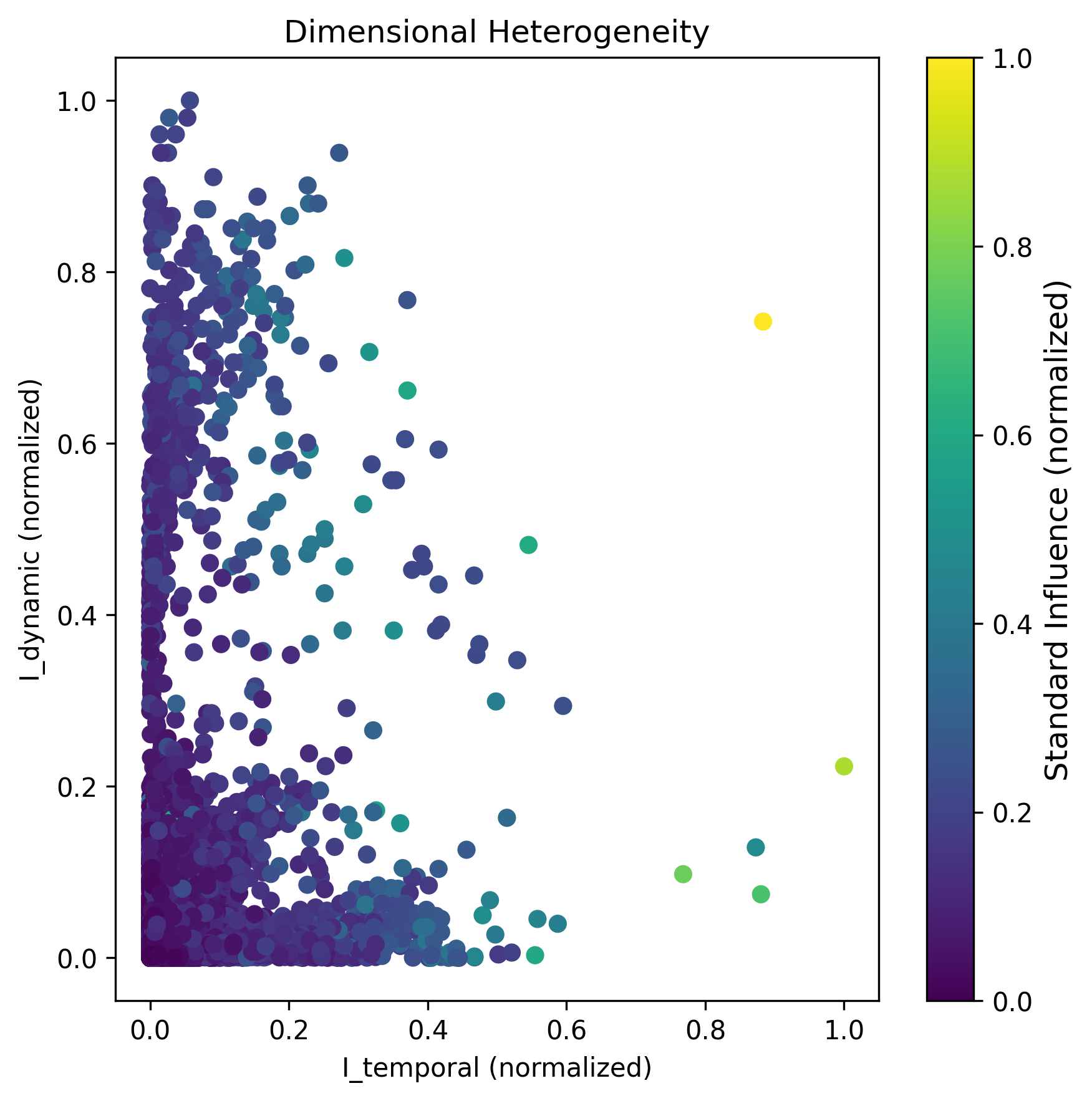}
}

\vspace{0.8em}

\subfloat[Temporal vs. Factual]{
    \includegraphics[width=0.32\textwidth]{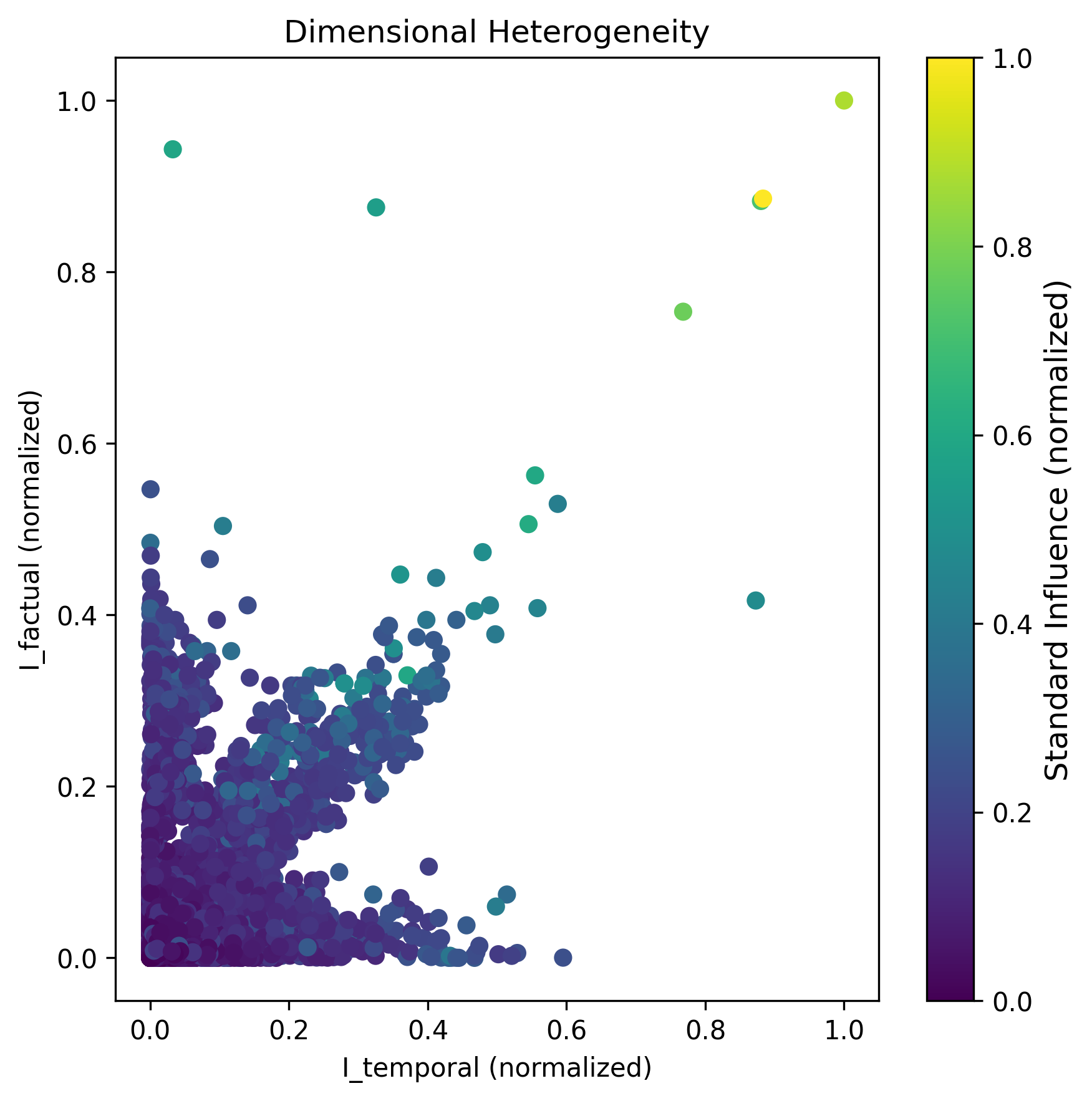}
}
\hfill
\subfloat[Temporal vs. Text]{
    \includegraphics[width=0.32\textwidth]{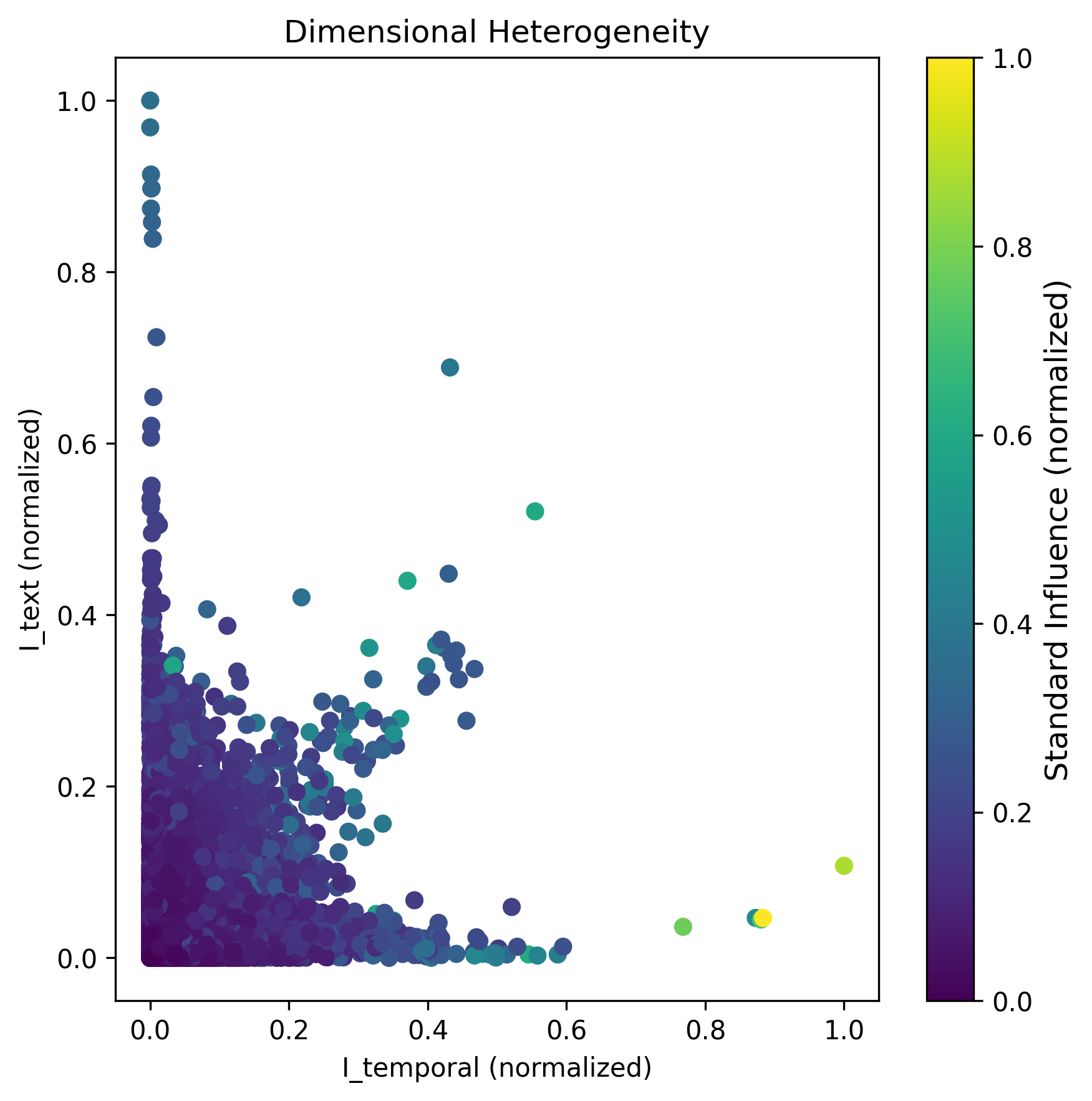}
}
\hfill
\subfloat[Text vs. Factual]{
    \includegraphics[width=0.32\textwidth]{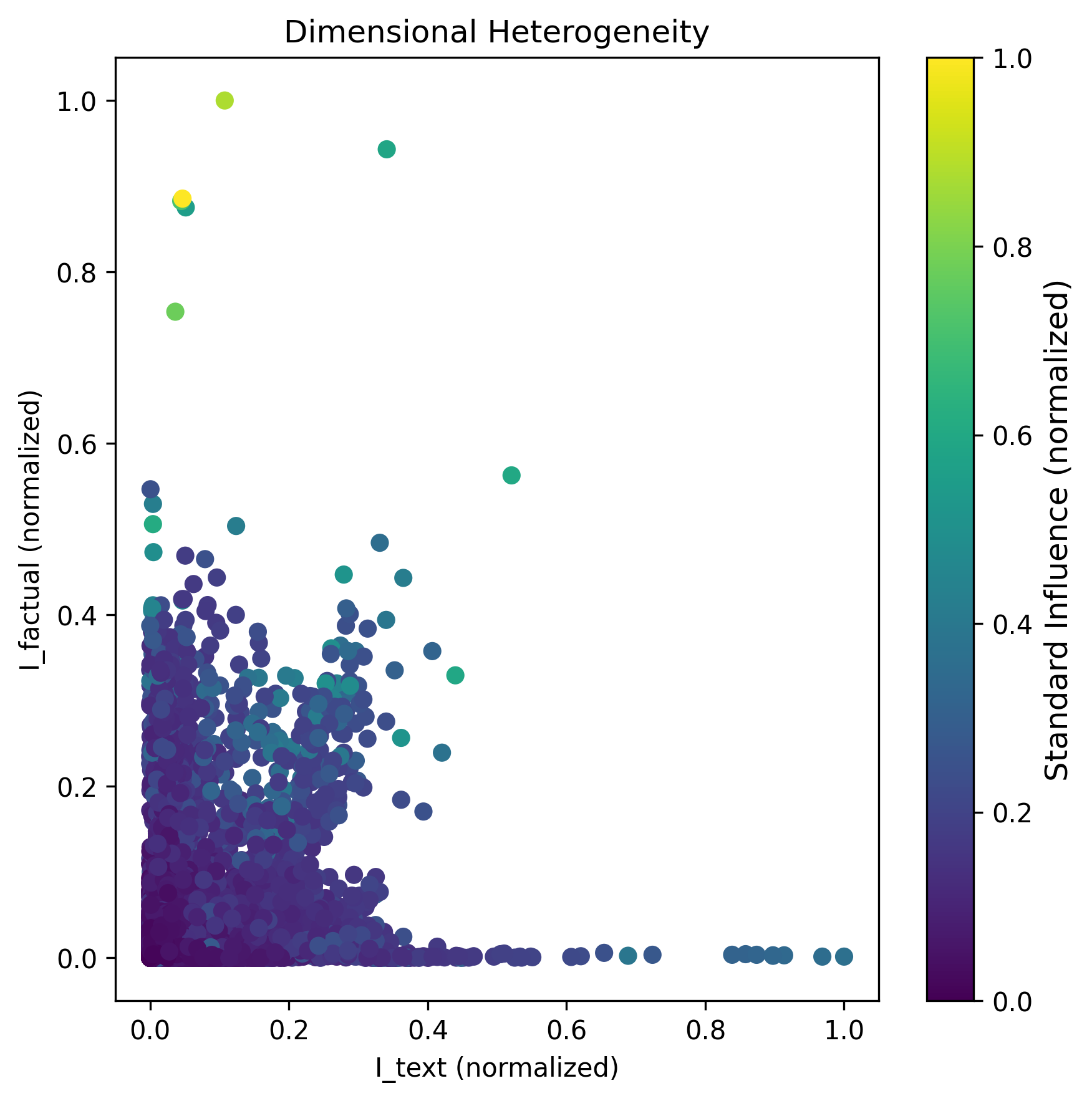}
}

\caption{Scatter plot comparisons across dimensions.}
\label{fig:appendix_scatter}
\end{figure*}

\subsection{Additional Analysis on DDR Temperature}
\label{app:ddr_temperature}

We further study the sensitivity of DDR to the temperature hyperparameter $\tau$ in Eq.~\ref{eq:sigmoid_weight}.
Specifically, we evaluate $\tau \in \{0.1, 0.5, 1.0, 5.0\}$ while keeping all other training and evaluation settings unchanged.
As shown in Table~\ref{tab:tau_sensitivity}, DDR remains relatively stable across a wide range of $\tau$ values.
In our main experiments, we use $\tau=1.0$ as a fixed default value without extensive tuning.
Since $\tau$ is not selected on the test set and is shared across all reward dimensions, this design avoids test-set leakage and reduces dimension-specific hyperparameter tuning.

\begin{table}[h]
    \centering
    \caption{
    Sensitivity analysis of the DDR temperature $\tau$.
    Spearman correlations ($\rho \times 100$) are reported across five reward dimensions.
    DDR is relatively stable across a wide range of $\tau$ values.
    }
    \label{tab:tau_sensitivity}
    \resizebox{0.82\linewidth}{!}{
    \begin{tabular}{c ccccc}
        \toprule
        $\boldsymbol{\tau}$ & \textbf{Visual} & \textbf{Temporal} & \textbf{Dynamic} & \textbf{Text} & \textbf{Factual} \\
        \midrule
        0.1 & 82.54 & 74.36 & 56.59 & 49.56 & 79.84 \\
        0.5 & 82.15 & 74.84 & 57.17 & 52.38 & 80.19 \\
        1.0 & 82.29 & 74.91 & 57.61 & 52.72 & 80.12 \\
        5.0 & 82.21 & 75.62 & 57.02 & 52.27 & 80.36 \\
        \bottomrule
    \end{tabular}
    }
\end{table}

\subsection{Stability Across Multiple Runs}
\label{app:stability}

To evaluate the stability of our method, we report the standard deviation over three independent runs.
As shown in Table~\ref{tab:std_results}, both DDR and DDP exhibit comparable or smaller variance than the baseline across most reward dimensions.
Together with the main results in Table~\ref{tab:main_results_merged}, this suggests that the improvements are stable across runs rather than being caused by random variation.

\begin{table}[h]
    \centering
    \caption{
    Standard deviation over three independent runs.
    We report the standard deviation of Spearman correlations ($\rho \times 100$) across five reward dimensions.
    }
    \label{tab:std_results}
    \resizebox{0.62\linewidth}{!}{
    \begin{tabular}{lccccc}
        \toprule
        \textbf{Method} & \textbf{VQ} & \textbf{TC} & \textbf{DD} & \textbf{TVA} & \textbf{FC} \\
        \midrule
        Baseline & 0.52 & 0.29 & 0.83 & 0.53 & 0.44 \\
        DDR      & 0.32 & 0.24 & 0.61 & 0.43 & 0.43 \\
        DDP      & 0.41 & 0.36 & 0.53 & 0.46 & 0.26 \\
        \bottomrule
    \end{tabular}
    }
\end{table}

\subsection{Validation-Based Pruning Ratio Selection}
\label{app:pruning_ratio_selection}

Since the public dataset does not provide an official validation split, we construct a held-out validation set by randomly sampling 40\% of the original test split for hyperparameter selection.
We select the pruning ratio for the main results in Table~\ref{tab:main_results_merged} according to the best validation performance.

\begin{table}[h]
    \centering
    \caption{
    Validation performance under different pruning ratios.
    Spearman correlations ($\rho \times 100$) are reported across five reward dimensions.
    The pruning ratios used in Table~\ref{tab:main_results_merged} are selected based on this validation study.
    }
    \label{tab:valid_pruning_ratio}
    \resizebox{0.62\linewidth}{!}{
    \begin{tabular}{c ccccc}
        \toprule
        \textbf{Ratio (\%)} & \textbf{VQ} & \textbf{TC} & \textbf{DD} & \textbf{TVA} & \textbf{FC} \\
        \midrule
        0.1 & \textbf{81.10} & 74.18 & 56.33 & \textbf{52.55} & 79.36 \\
        0.5 & 80.74 & \textbf{75.46} & \textbf{57.84} & 51.57 & \textbf{79.97} \\
        1.0 & 80.56 & 73.88 & 57.37 & 50.62 & 77.19 \\
        \bottomrule
    \end{tabular}
    }
\end{table}

As shown in Table~\ref{tab:valid_pruning_ratio}, the validation results suggest a smaller pruning ratio for VQ, TVA and a larger ratio for TC, DD, and FC.
This trend is consistent with the sensitivity analysis in Fig.~\ref{fig:pruning_ratio}.

\subsection{Details on the Off-Diagonal Ablation}
\label{app:off_diagonal_param_space}

We clarify the parameter space used in the off-diagonal ablation.
In our main method, we compute dimension-specific self-influence using only the final linear regression head.
This choice is motivated by both efficiency and the structure of MVRMs: each reward dimension has an independent output vector in the final head, making the diagonal term a direct and efficient proxy for the supervision risk of the corresponding sample-dimension pair.
Under this strictly last-layer parameter space, and with the identity-Hessian approximation, the off-diagonal entries between independent output heads are analytically zero.

The off-diagonal ablation is therefore computed in a larger parameter space.
Specifically, we use the parameters of the last two layers, which include shared representation parameters before the final dimension-specific output heads.
In this setting, gradients from different reward dimensions are coupled through shared parameters, so the off-diagonal influence terms are generally non-zero even under the identity-Hessian approximation.
However, this requires additional backpropagation through the shared layer and is therefore less efficient than our final-head diagonal approximation.

The results in Table~\ref{tab:row_sum_ablation} show that incorporating these off-diagonal terms through row-sum aggregation does not improve performance over the diagonal-only method.
This suggests that the additional cross-dimensional coupling captured by the larger parameter space is not a more reliable signal for data refinement.
Therefore, our diagonal-only design provides a favorable accuracy--efficiency trade-off: it avoids unnecessary backward computation through shared layers while preserving the dimension-specific supervision-risk signal needed for pruning and reweighting.

\subsection{Overlap Analysis of Dimension-Specific High-Risk Sets}
\label{app:ddp_overlap}

DDP removes a sample if it falls into the high-risk set of any reward dimension.
One concern is that this union-based criterion may be overly aggressive when multiple dimensions are considered.
To analyze this, we measure the cumulative removal ratio as we progressively include more reward dimensions.
For each dimension, we select the top 0.5\% samples with the highest dimension-specific self-influence.

\begin{table}[!h]
    \centering
    \caption{
    Overlap analysis of dimension-specific high-risk sets.
    We report the cumulative removal ratio when progressively taking the union of top 0.5\% high-self-influence samples from more dimensions.
    If all dimensions selected identical samples, the removal ratio would remain 0.5\%; if the sets were fully disjoint, it would reach 2.5\% for five dimensions.
    }
    \label{tab:ddp_overlap}
    \resizebox{0.62\linewidth}{!}{
    \begin{tabular}{c ccccc}
        \toprule
        \textbf{\# Dimensions Included} & 1 & 2 & 3 & 4 & 5 \\
        \midrule
        \textbf{Cumulative Removal Ratio (\%)} & 0.50 & 0.88 & 1.37 & 1.82 & 2.14 \\
        \bottomrule
    \end{tabular}
    }
\end{table}

As shown in Table~\ref{tab:ddp_overlap}, the cumulative removal ratio increases from 0.50\% to 2.14\% as more dimensions are included.
This value is lower than the 2.50\% removal ratio expected if all dimension-specific high-risk sets were fully disjoint, indicating that there is partial overlap across dimensions.
At the same time, the removal ratio is substantially higher than 0.50\%, suggesting that different reward dimensions also identify unique high-risk samples.

This result supports the motivation of dimension-disentangled pruning.
If all dimensions selected almost the same samples, global pruning would be sufficient.
Instead, the observed partial overlap indicates that each dimension contributes additional high-risk samples, while the total removal ratio remains moderate.
Therefore, the union-based DDP criterion captures dimension-specific supervision risk without excessively discarding the training data.

\subsection{Combination of DDP and DDR.}
We also evaluate combining DDP and DDR, but do not observe significant additional gains compared with using either strategy alone.
This is likely because both methods are derived from the same underlying dimension-specific self-influence signal and therefore act on largely overlapping high-risk samples.
As a result, their effects are not strongly complementary in practice.

\subsection{Controlled Label-Noise Detection}
\label{app:controlled_label_noise}

A key question is whether high dimension-specific self-influence is associated with true label noise or merely reflects difficult but correct samples.
To directly examine this issue, we conduct a controlled label-noise detection experiment.
Specifically, for each reward dimension, we randomly select 10\% of the training samples and corrupt only the labels of that dimension while keeping the labels of all other dimensions unchanged.
This creates a known set of dimension-specific anomalous supervision signals, allowing us to quantitatively evaluate whether self-influence can recover the injected label noise.

After label corruption, we compute the dimension-specific self-influence score $\mathcal{S}_k(z)$ for each sample and use it as an anomaly score for detecting corrupted labels in the $k$-th reward dimension.
We evaluate detection performance using AUROC, where corrupted samples are treated as positives and unmodified samples as negatives.
This setup isolates label noise from naturally hard but correct samples, since the corrupted set is known by construction.

\begin{table}[t]
    \centering
    \caption{
    Controlled label-noise detection results.
    We randomly corrupt 10\% of training labels for each reward dimension and use dimension-specific self-influence as the anomaly score.
    AUROC is reported for detecting the injected label noise.
    }
    \label{tab:controlled_label_noise}
    \resizebox{0.72\linewidth}{!}{
    \begin{tabular}{lccccc}
        \toprule
        \textbf{Metric} & \textbf{Visual} & \textbf{Temporal} & \textbf{Dynamic} & \textbf{Text} & \textbf{Factual} \\
        \midrule
        AUROC & 99.04 & 98.35 & 98.25 & 95.64 & 97.59 \\
        \bottomrule
    \end{tabular}
    }
\end{table}

As shown in Table~\ref{tab:controlled_label_noise}, dimension-specific self-influence achieves AUROC scores above 95 across all five reward dimensions.
These results indicate that high self-influence is highly effective at detecting injected dimension-specific label anomalies.
Importantly, this experiment does not claim that every naturally high-influence sample is label noise; difficult but correct samples may still receive high influence scores.
Rather, the results provide controlled evidence that when true label noise exists, dimension-specific self-influence can reliably assign higher anomaly scores to corrupted supervision signals.
This supports our use of self-influence as a supervision-risk signal for prioritizing samples for pruning, reweighting, or human inspection.

\subsection{Additional Label Noise Visualization}
\label{app:vis_more}

In this section, we provide additional qualitative examples of label noise to complement the analysis in the main text. Specifically, we visualize a broader set of training samples that exhibit high T2V Alignment dimension-specific self-influence but low visual quality self-influence.

All samples are drawn from the same training split and evaluated using the trained MVRM. For each example, the ground truth is originally labeled as Score~1; however, upon reevaluation, we find that the text is in fact aligned with the video.

\textbf{Analysis: Conflict between Annotation Protocols and VLM Priors.}
We attribute the high sef-influence of these samples to the fundamental conflict between rigid dataset annotation protocols and the open-world semantic knowledge inherent in the VLM-initialized reward model.
Since the reward model is initialized from a powerful VLM model\citep{jiangmantis}, it correctly perceives the semantic connection. However, the loss function enforces convergence to the incorrect Score 1 label, generating substantial self-influence.

\subsection{Limitation and Future Work} 
\label{limitation}
While our method successfully disentangles data quality at the objective level (e.g., distinguishing visual quality from semantic alignment), it currently operates on the holistic representation of each video sample. 
Specific defects often occur locally—such as transient artifacts in a few frames or localized spatial distortions—which might not be fully captured by a video-level influence score. 
Looking ahead, a primary focus of our research will be extending this analysis to the \textit{spatio-temporal domain}. We believe that developing mechanisms to identify and refine data at the frame or patch level will yield even more granular control and further advance the precision of video generation alignment.

\section{Broader Impacts}
\label{app:broader_impacts}

This work studies data refinement for multidimensional video reward models, aiming to improve the reliability of fine-grained video generation evaluation. 
A positive societal impact is that more reliable reward models may help identify low-quality, unsafe, or poorly aligned generated videos, supporting better evaluation and monitoring of video generation systems. 
Our dimension-wise analysis may also help reveal annotation inconsistencies and improve the quality of human feedback datasets.

At the same time, improved reward models may indirectly benefit stronger video generation systems, which could be misused to produce misleading, harmful, or deceptive synthetic media. 
In addition, if the training annotations contain social or cultural biases, refining reward models based on these annotations may preserve or amplify such biases. 
Incorrect reward estimates may also lead to overconfidence in automated evaluation results.

To mitigate these risks, we recommend using multidimensional reward models as auxiliary evaluation tools rather than sole decision makers, auditing annotations across demographic and cultural contexts, and combining automated reward scores with human review for high-stakes applications. 
Future deployment of such models should include monitoring for biased behavior, misuse, and distribution shifts.

\clearpage
\begin{figure}[h]
    \centering

    \begin{subfigure}{0.45\textwidth}
        \includegraphics[width=\textwidth]{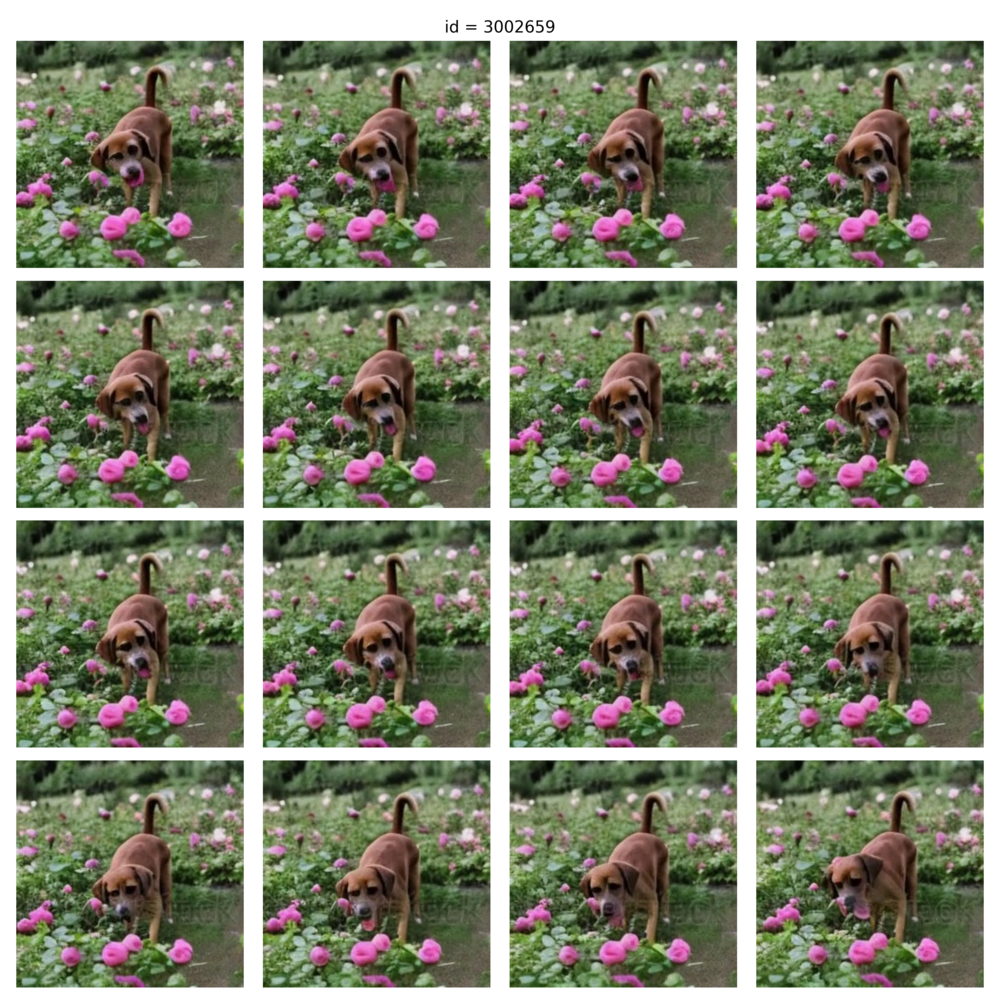}
        \twolinecap{cinematic shot of a cute dog playing in a rose garden.}
    \end{subfigure}
    \hfill
    \begin{subfigure}{0.45\textwidth}
        \includegraphics[width=\textwidth]{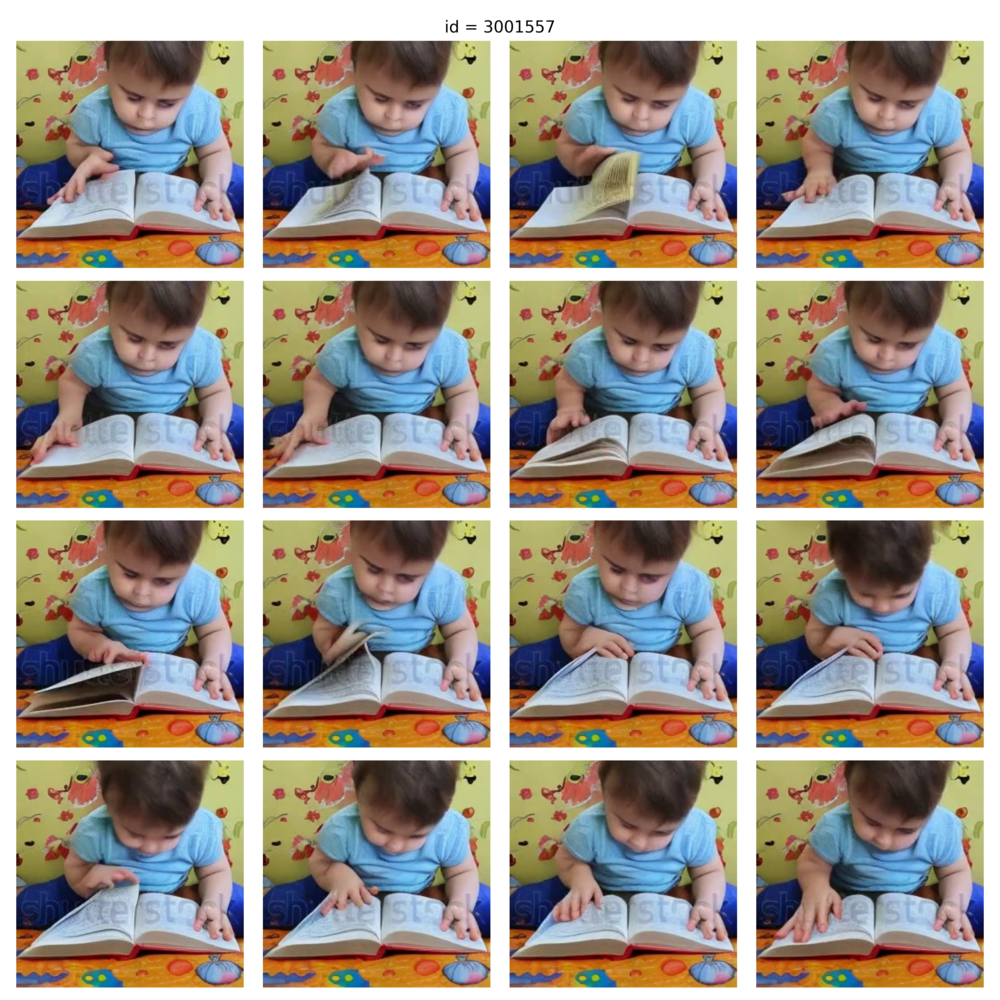}
        \twolinecap{a child play with a book.}
    \end{subfigure}

    \vspace{1em}

    \begin{subfigure}{0.45\textwidth}
        \includegraphics[width=\textwidth]{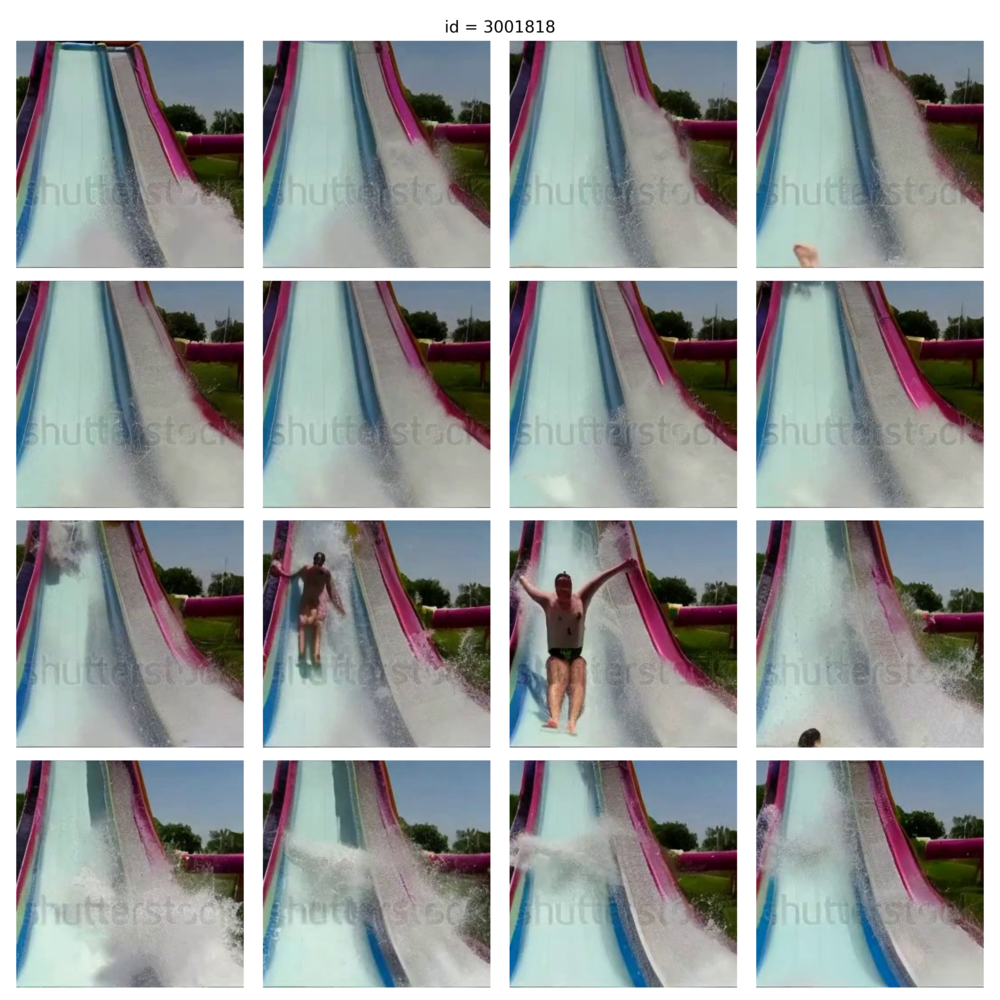}
        \twolinecap{a person going down a fast water slide.}
    \end{subfigure}
    \hfill
    \begin{subfigure}{0.45\textwidth}
        \includegraphics[width=\textwidth]{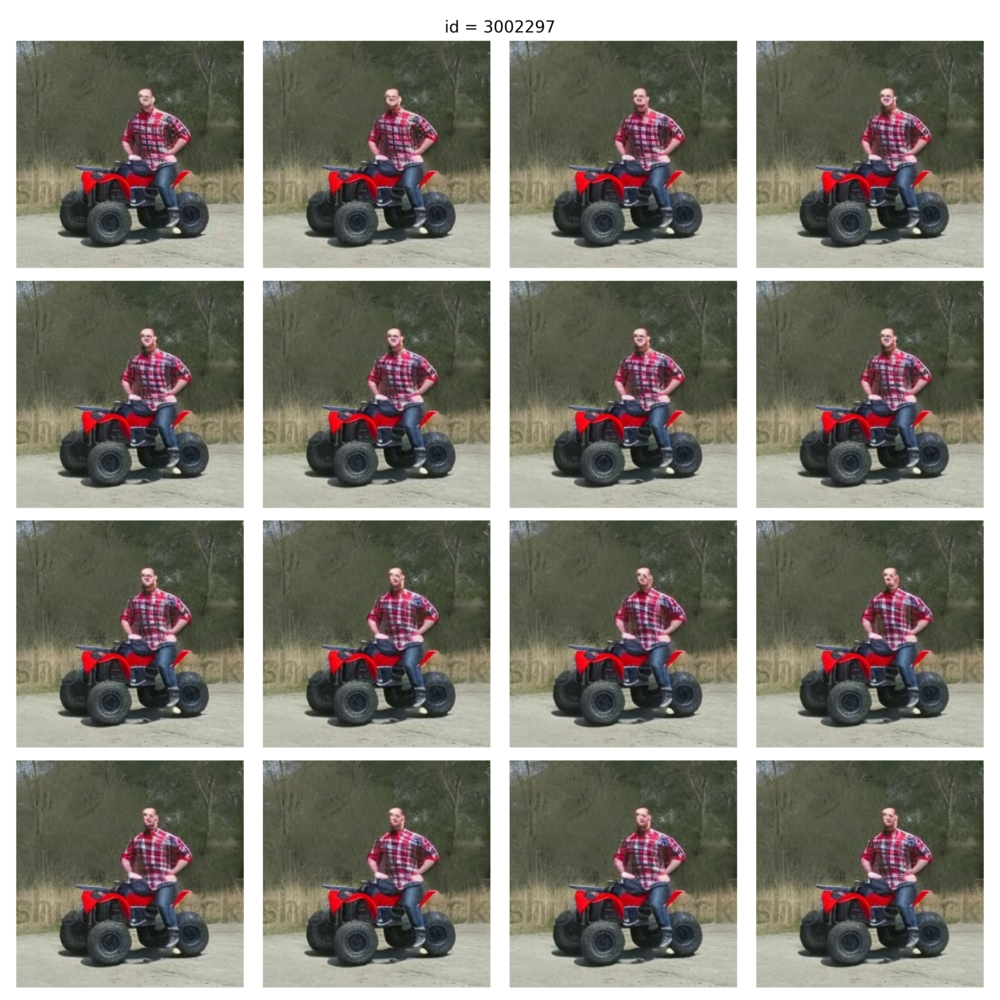}
        \twolinecap{a man sitting on a quad.}
    \end{subfigure}
 
    \caption{Label noise visualization.}
\end{figure}
\clearpage
    \begin{figure}[h]
    \centering
    \begin{subfigure}{0.45\textwidth}
        \includegraphics[width=\textwidth]{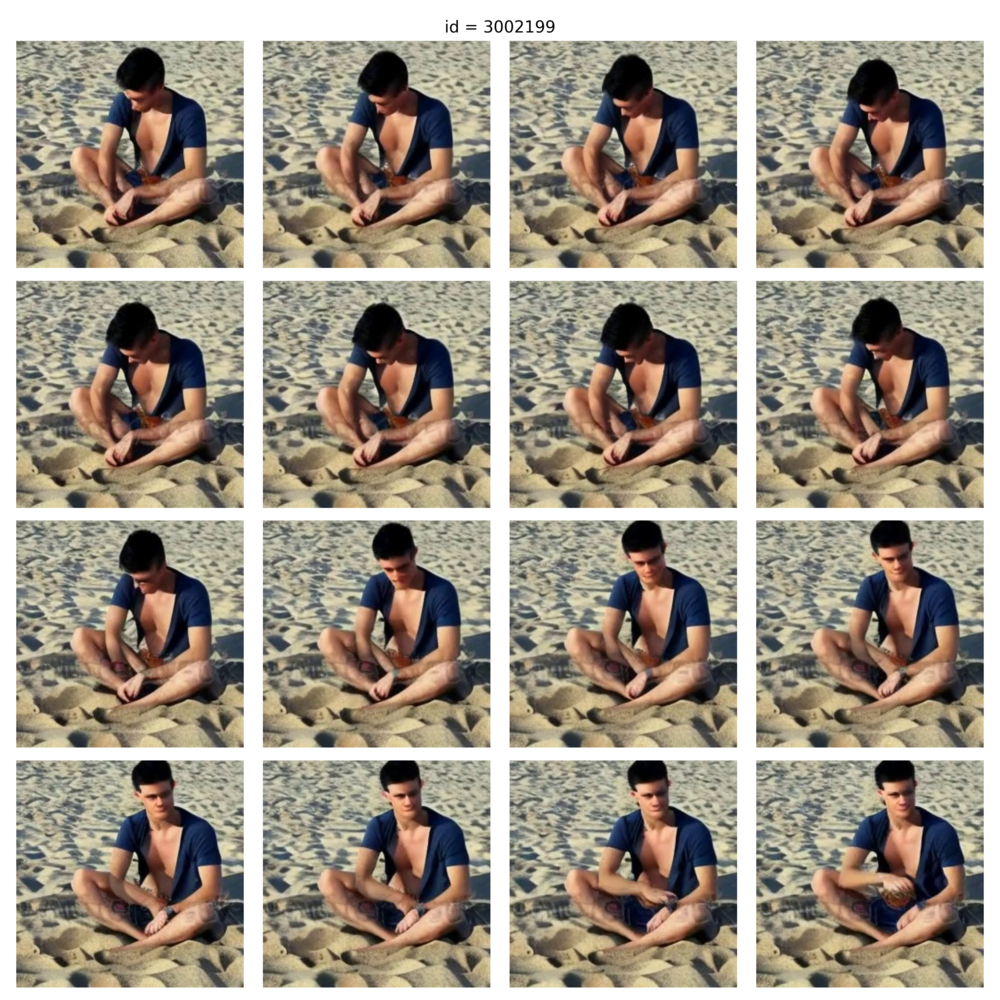}
        \twolinecap{handsome guy sitting on the sand.}
    \end{subfigure}
    \hfill
    \begin{subfigure}{0.45\textwidth}
        \includegraphics[width=\textwidth]{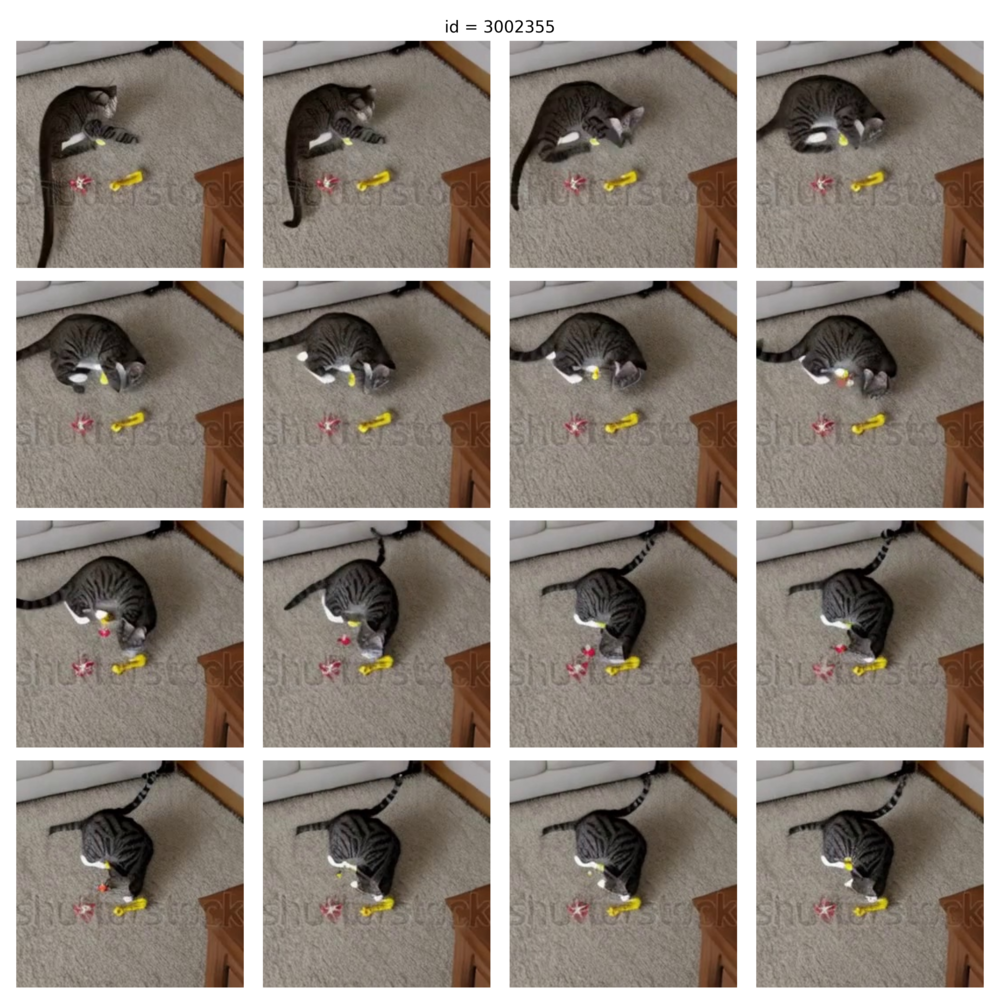}
        \twolinecap{a realistic adorable kitty playing with a cat toy in a brightly lit living room.}
    \end{subfigure}
 
    \vspace{1em}
    \begin{subfigure}{0.45\textwidth}
        \includegraphics[width=\textwidth]{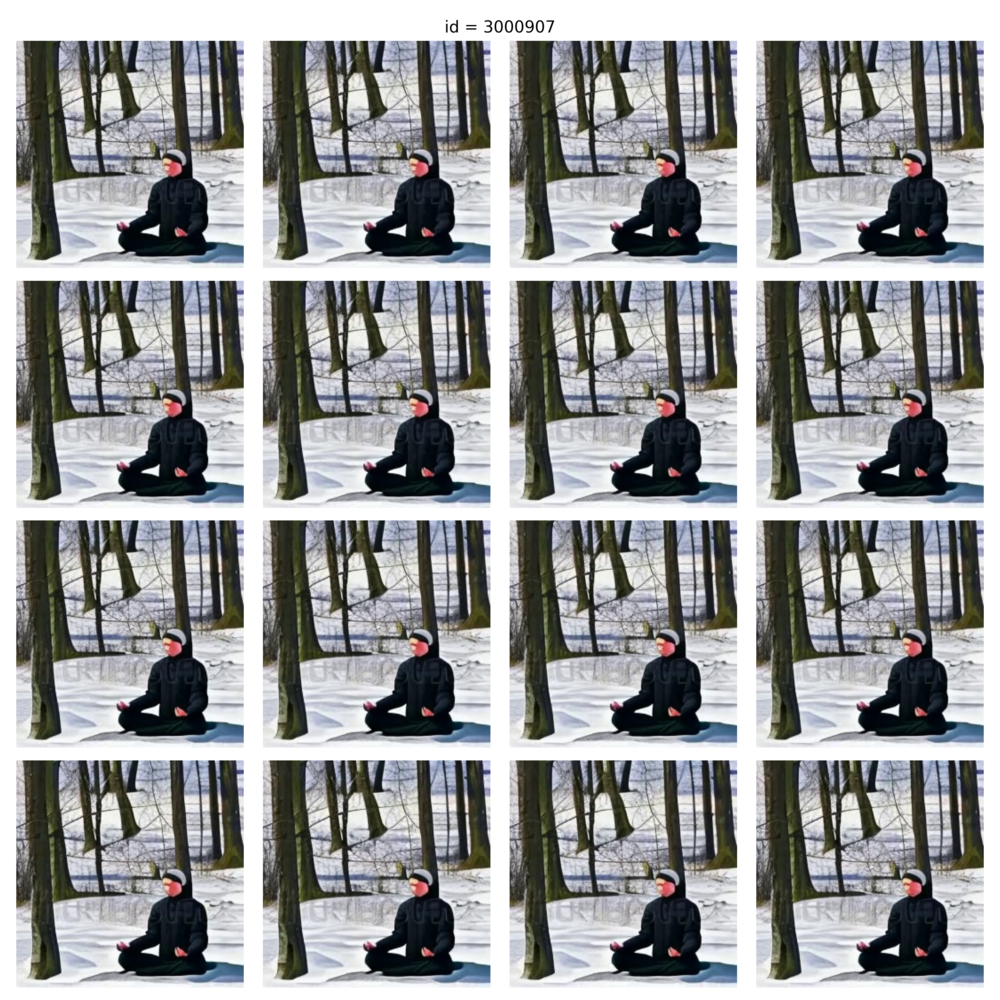}
        \twolinecap{Meditating man on snow during the day, realistic art, 9:16.}
    \end{subfigure}
    \hfill
    \begin{subfigure}{0.45\textwidth}
        \includegraphics[width=\textwidth]{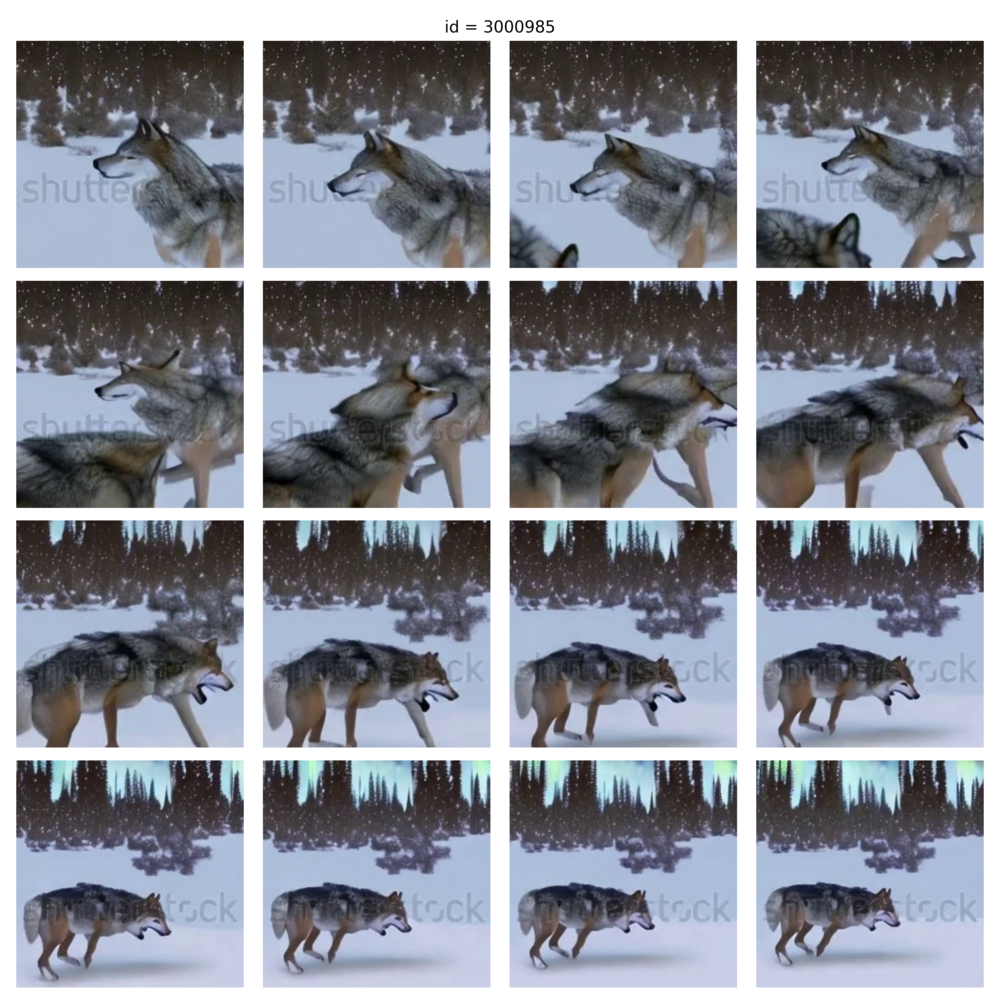}
        \twolinecap{running wolf, northern lights, snow, zoom out, realistic, 4k.}
    \end{subfigure}
    \caption{Label noise visualization.}
\end{figure}
\clearpage


\end{document}